\documentclass{article}

\usepackage{nonatbib}
\usepackagewithoutnatbib{jmlr2e}
\usepackage{fixthmtools}
\usepackage{lastpage}

\usepackage[centerfigures,citations,authoryear,commands,enumerate,environments,lmrscinnershape,theorems]{AVT}

\let\UrlFont\scshape

\usepackage[font=normalsize]{subcaption}

\usepackage{algorithm}
\usepackage[noend]{algorithmic}
\usepackage{xcolor}
\usepackage{siunitx}
\usepackage{array}

\bibliography{citations}

\Crefname{ALC@unique}{Line}{Lines}
\newcounter{myalg}
\AtBeginEnvironment{algorithmic}{\refstepcounter{myalg}}
\expandafter\csname @addtoreset\expandafter\endcsname{ALC@unique}{myalg}

\jmlrheading{25}{2024}{1-\pageref{LastPage}}{10/22; Revised
9/23}{1/24}{22-1170}{Alexander Terenin, David R. Burt, Artem Artemev, Seth Flaxman, Mark van der Wilk, Carl Edward Rasmussen, and Hong Ge}
\ShortHeadings{Numerically Stable Sparse Gaussian Processes}
{Terenin, Burt, Artemev, Flaxman, van der Wilk, Rasmussen, and Ge}
\firstpageno{1}

\title{Numerically Stable Sparse Gaussian Processes\\via Minimum Separation using Cover Trees}

\author{\name Alexander Terenin\textsuperscript{\ensuremath{*}} \\
\addr University of Cambridge and Imperial College London
\AND
\name David R. Burt\textsuperscript{\ensuremath{*}} \\
\addr University of Cambridge and MIT
\AND
\name Artem Artemev\textsuperscript{\ensuremath{*}} \\
\addr Imperial College London and Secondmind
\AND
\name Seth Flaxman \\
\addr University of Oxford
\AND
\name Mark van der Wilk \\
\addr Imperial College London and University of Oxford
\AND
\name Carl Edward Rasmussen \\
\addr University of Cambridge and Secondmind
\AND
\name Hong Ge \\
\addr University of Cambridge
}

\editor{Mohammad Emtiyaz Khan}

\begin{document}

\maketitle

\begin{table}[b!]
\vspace*{-1.5ex}
\raggedright
\footnoterule
\footnotesize\textsuperscript{\ensuremath{*}}Equal contribution.
\\
Code available at: \url{https://github.com/awav/conjugate-gradient-sparse-gp}.
\end{table}

\begin{abstract}
Gaussian processes are frequently deployed as part of larger machine learning and decision-making systems, for instance in geospatial modeling, Bayesian optimization, or in latent Gaussian models.
Within a system, the Gaussian process model needs to perform in a stable and reliable manner to ensure it interacts correctly with other parts of the system.
In this work, we study the numerical stability of scalable sparse approximations based on inducing points. 
To do so, we first review numerical stability, and illustrate typical situations in which Gaussian process models can be unstable.
Building on stability theory originally developed in the interpolation literature, we derive sufficient and in certain cases necessary conditions on the inducing points for the computations performed to be numerically stable.
For low-dimensional tasks such as geospatial modeling, we propose an automated method for computing inducing points satisfying these conditions.
This is done via a modification of the cover tree data structure, which is of independent interest. 
We additionally propose an alternative sparse approximation for regression with a Gaussian likelihood which trades off a small amount of performance to further improve stability.
We provide illustrative examples showing the relationship between stability of calculations and predictive performance of inducing point methods on spatial tasks.
\end{abstract}

\section{Introduction}

Gaussian processes are a flexible framework and model class for learning unknown functions.
By way of being constructed in the language of Bayesian learning, Gaussian process models provide an ability to incorporate prior information into the model, and assess and propagate uncertainty in a principled manner.
This makes them well-suited for a wide variety of areas where these capabilities are important, including statistical applications such as spatial modeling \cite{cressie92}, and decision-making applications such as Bayesian optimization \cite{snoek12}, sensor placement \cite{krause08}, and active learning \cite{krause07}.

In many settings, the increased availability of data and need to accurately model higher-resolution phenomena has led to a strong interest in working with Gaussian processes at a larger scale.
Unfortunately, classical Gaussian process models generally scale cubically with training data size due to the need to solve large linear systems of equations.
This mismatch has led to a longstanding and fruitful line of work on scalable Gaussian processes.
In the era of GPUs and automatic differentiation, two classes of scalable approximations have been deployed within major Gaussian process software packages, including \emph{GPflow} \parencite{gpflow} and \emph{GPyTorch} \parencite{gardner2018gpytorch}: those based on \emph{inducing point methods} \parencite{quinonero2005unifying, titsias09, hensman13}, and iterative methods such as the \emph{(preconditioned) conjugate gradient} algorithm \parencite{gibbs1997efficient, gardner2018gpytorch}.

Motivated by these advances, in this work we study a complementary question: how can we \emph{guarantee} that the above algorithms run successfully, no matter what kind of data they are provided with?
This question is particularly important in areas such as Bayesian optimization where the data is not available in advance, and in latent variable models, where the data going into the Gaussian process is reconstructed from auxiliary observations.
We focus chiefly on inducing point methods applied to geospatial or other low-dimensional data, though the techniques we develop also reveal insights on how the performance of conjugate-gradient-based approaches depends on data.

Our work complements existing analyses of numerical stability in Gaussian processes.
\textcite{foster09} study the subset of regressors approximation, and describe in detail which numerical linear algebra routines one should use to implement it in the most stable manner.
\textcite{basak21} study maximum likelihood estimation, observe that stability of log-determinant computations is effectively controlled by the same quantities that control stability of linear solves, and describe practical strategies for initializing and stopping numerical optimization in a stable manner.
In contrast, we describe how different Gaussian process approximations and their hyperparameters affect the numerical stability of the linear-algebraic routines which need to be implemented, regardless of the details of how this implementation is done. 

Our contributions are as follows.
We (i) identify that mathematical tools developed in the interpolation literature \cite{narcowich92, narcowich1994condition, schaback95, diederichs2019improved} can be used to show, in a large class of Gaussian process priors, that datasets satisfying a \emph{minimum separation} criterion lead to linear systems which are provably stable independent of data size.
Using this insight, we (ii) develop a clustering analogy and propose a technique to guarantee stability in sparse approximations by using a refinement of the \emph{cover tree} data structure \cite{beygelzimer06} to automatically select inducing points that satisfy the needed criterion while staying as close to the true data as possible.
This approximation is parameterized by a hyperparameter called the \emph{spatial resolution}, which directly controls the tradeoff between approximation accuracy and computational cost.
Following this, we (iii) extend the clustering analogy further, and propose an alternative, non-variational inducing point approximation which trades off a small amount of performance to further improve stability.
Finally, we (iv) study error and computational complexity of the proposed methods, and evaluate them on a number of examples.

\section{Gaussian Processes}

Let $X$ be a set.
We say that a random function $f : X \-> \R$ is a \emph{Gaussian process} if, for any finite set of points $\v{x} \in X^N$, the random vector $f(\v{x})$ is multivariate Gaussian.
We write $f \~[GP](\mu, k)$, where $\mu(\v{x}) = \E(f(\v{x}))$ is the \emph{mean function}, and $k(\v{x},\v{x}') = \Cov(f(\v{x}),f(\v{x}'))$ is the \emph{covariance kernel}, which determine the mean and covariance of the corresponding multivariate Gaussians.

Define the likelihood $\v{y} = f(\v{x}) + \v\eps$ where $\v\eps \~[N](\v{0},\m\Sigma)$. 
Given training data $\v{x},\v{y}$, if we place a Gaussian process prior $f\~[GP](0,k)$, then the posterior $f\given\v{y}$ is a Gaussian process with
\[
\E(f\given\v{y}) &= \m{K}_{(\.)\v{x}}(\m{K}_{\v{x}\v{x}} + \m\Sigma)^{-1}\v{y}
&
\Cov(f\given\v{y}) &= \m{K}_{(\.,\.')} - \m{K}_{(\.)\v{x}}(\m{K}_{\v{x}\v{x}} + \m\Sigma)^{-1}\m{K}_{\v{x}(\.')}
\]
where $(\.)$ and $(\.')$ denote arbitrary sets of test locations and for sets $a, b$ we use $\m{K}_{ab}$ to denote the kernel matrix formed by evaluating $k$ at points in $a \times b$.
We assume throughout that $\m\Sigma$ is diagonal.
It is also possible to express the posterior in the form of \emph{pathwise conditioning} \cite{wilson20,wilson21} as 
\[
(f\given \v{y})(\.) = f(\.) + \m{K}_{(\.)\v{x}} (\m{K}_{\v{x}\v{x}} + \m\Sigma)^{-1}(\v{y} - f(\v{x}) - \v\eps)
\]
where equality holds in distribution.
These expressions describe the true posterior, whose computational costs are, classically, $\c{O}(N^3)$ owing to the Cholesky decomposition used to handle the matrix inverse.
To alleviate this, we now consider approximations.

\subsection{Inducing Points}
Many of the most widely-used Gaussian process approximations are based on the idea of \emph{inducing points} \cite{csato2002sparse,seeger2003fast,quinonero2005unifying}.
We present the variational formulation of \textcite{titsias09, hensman13}, using a pathwise conditioning construction \cite{wilson20,wilson21}.
Let $\v{z} \in X^M$, and define the \emph{variational posterior}
\[
\label{eq:post-hensman}
(f\given \v{u})(\.) &= f(\.) + \m{K}_{(\.)\v{z}} \m{K}_{\v{z}\v{z}}^{-1}(\v{u} - f(\v{z}))
&
\v{u} &\~[N](\v{m}, \m{S})
\]
where $\v{m}$, $\m{S}$ and $\v{z}$ are approximation parameters.
We find these by minimizing the Kullback--Leibler divergence of the approximation from the true posterior. 
By the chain rule for Kullback--Leibler divergences, this quantity reduces to an easy-to-evaluate divergence between finite-dimensional Gaussians \cite{matthews16}.
Given $\v{z}$, which one can interpret as a learned set of approximate $\v{x}$-values, the optimal values for $\v{m}$ and $\m{S}$ can be solved for analytically in closed form: see \textcite{titsias09} for~details.

This approximation only requires factorizing a smaller kernel matrix, leading to a cost of $\c{O}(M^3)$, a significant improvement if $M \ll N$.
Following \textcite{hensman13}, the approximation also inherits variational methods' compatibility with stochastic optimization and mini-batch methods.

\subsection{Gaussian Processes for Geospatial Modelling}
Gaussian processes models are widely used in spatial statistics both for interpolation, and to allow for spatially correlated noise in spatial regression models \cite{cressie92}. In low-dimensional settings, such as geospatial data, inducing point approximations can lead to provably near-linear computational scalability, particularly when input points, $x_i$, are sampled from within a bounded region, and the Gaussian process prior is smooth enough---see \textcite{burt19,burt2020convergence}. 
This makes them a reasonable choice for many applications.

In geospatial modeling, specialized methods including multi-resolution approaches \cite{katzfuss2017multi} or sparse-precision approaches \cite{vecchia1988,lindgren11,Katzfuss2021_vecchia} often result in accurate approximations, especially when the domain is much larger than the scale over which the function being modelled varies. 
See for instance \textcite{heaton2019case} for a comparison of a large number of approaches on problems in this class.
In such settings, low-rank calculations very similar to those used in inducing point methods remain central to the construction of a number of more tailored approaches, including in particular multi-resolution methods \cite{katzfuss2017multi} and combinations of low-rank and compactly supported methods \cite{sang2012fsa}. 
A better understanding of numerical stability in inducing point methods is therefore an important step to better understanding when these more sophisticated methods designed specifically for spatial data will be numerically stable.

\subsection{Numerical Stability}

On a computer, the calculations needed to deploy a Gaussian process, whether exact or approximate, must be performed in floating-point arithmetic.
Due to the introduction of roundoff error, algorithms which solve linear systems $\m{A}^{-1}\v{b}$ can fail if the system's solution is too sensitive to the numerical values of $\m{A}$ or $\v{b}$. We assume throughout that $\m{A}$ is symmetric positive definite, as will be the case for the kernel matrices of interest.
The key quantity used to understand how numerically stable a linear system is the associated \emph{condition number}
\[
\f{cond}(\m{A}) = \lim_{\eps\->0} \sup_{\norm{\v\delta}\leq\eps\norm{\v{b}}} \frac{\norm{\m{A}^{-1}(\v{b} + \v\delta) - \m{A}^{-1}\v{b}}_2}{\eps\norm{\m{A}^{-1}\v{b}}_2} = \norm[1]{\m{A}}_2 \norm[1]{\m{A}^{-1}}_2 = \frac{\lambda_{\max}(\m{A})}{\lambda_{\min}(\m{A})}
\]
of the matrix defining the system.
Here, $\lambda_{\max}$ and $\lambda_{\min}$ are the maximum and minimum eigenvalues, respectively, and $\norm{\.}_2$ denotes the Euclidean norm and the corresponding induced operator norm.
A linear system's condition number quantifies how difficult it is to solve numerically.
For a given floating-point precision, if $\f{cond}(\m{A})$ is small enough and the size of $\m{A}$ is not too large, then Cholesky factorization is guaranteed to succeed and return an accurate matrix square root.

\begin{result}
Let $\m{A}$ be a symmetric positive definite matrix of size $N \x N$.
Assume that $N > 10$, that
\[
\f{cond}(\m{A}) \leq \frac{1}{2^{-t} \x 3.9 N^{3/2}}
\]
where $t$ is the length of the floating point mantissa, and that $3N2^{-t}<0.1$.
Then floating point Cholesky factorization will succeed, producing a matrix $\m{L}$ satisfying
\[
\m{L}\m{L}^T &= \m{A} + \m{E}
&
\norm{\m{E}}_2 & \leq 2^{-t} \x 1.38 N^{3/2} \norm{\m{A}}_2
.
\]
\end{result}

\begin{proof}
This follows from \textcite[Corollary 2]{kielbasinski87}, and the relationship between the Euclidean and Frobenius matrix norms.
See also \textcite[Theorem 2]{wilkinson66}.
\end{proof}

For single precision arithmetic, we have $2^{-t} \approx 10^{-7.2}$, and for double-precision arithmetic, we have $2^{-t} \approx 10^{-16}$.
We therefore see that well-conditioned systems of equations lead to numerically stable Cholesky factorizations.
Moreover, when an iterative algorithm is used to solve a well-conditioned system, the iteration is often guaranteed to converge quickly, leading to computational benefits. 
In particular, if one solves the linear system with the conjugate gradient algorithm \cite{golub96}, then the algorithm is guaranteed to converge in logarithmically many steps.

\begin{result}
Let $\m{A}$ be a positive semi-definite matrix of size $N\x N$, and let $\eps > 0$.
Then, in exact arithmetic, the conjugate gradient algorithm for solving $\m{A}^{-1}\v{b}$ converges to within $\eps$ of the true solution, with respect to the norm induced by $\m{A}$, in 
\[
\c{O}\del{\sqrt{\f{cond}(\m{A})}\log \frac{\f{cond}(\m{A})\|\v{b}\|}{\eps}}
\]
total steps.
\end{result}

\begin{proof}
The claim follows from \textcite[Theorem 10.2.6]{golub96} by a direct calculation given in \Cref{apdx:convergence-of-cg}.
\end{proof}

Working with stable linear systems is therefore crucial for both inducing point methods which rely on Cholesky factorizations, and for algorithms based on conjugate gradients.
In the latter case, \emph{preconditioning} is often used to improve stability and accelerate convergence.
Preconditioning can be viewed as a way to partially solve the linear system, resolving its unstable components. 
State-of-the-art preconditioners in many areas of applied mathematics, such as multigrid methods for elliptic partial differential equations \cite{e11,xu17}, rely on detailed properties of their respective settings for construction and analysis.
A key step in designing provably effective preconditioners for Gaussian processes is therefore to understand how the model and data influence their resulting linear systems: we focus on this and defer ideas on preconditioning to future work.

\subsection{Instability in Gaussian Process Models}

\begin{figure}
\includegraphics{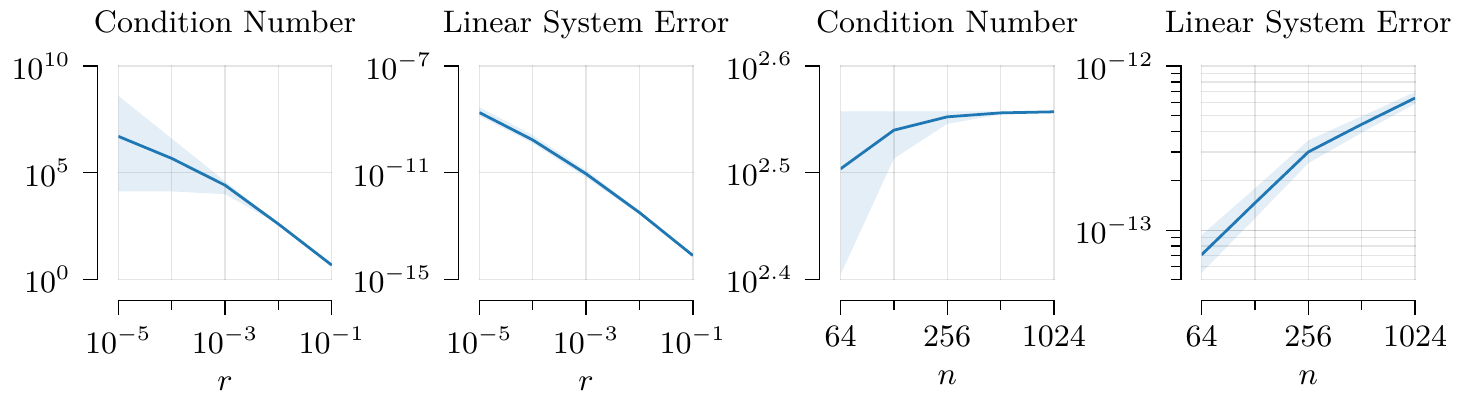}
\caption{Here we illustrate the condition numbers and resulting linear system error for the Kac--Murdock--Szegö matrix, which is the kernel matrix of an exponential kernel on a regularly-spaced one-dimensional grid. We vary the spacing $r$ and respective correlation $\rho = \exp(-r)$ of neighboring points, and the size $n$ of the matrix, with $\rho = 0.999$ used in cases where $n$ varies, and $n = 256$ used in cases where $\rho$ varies. We plot the condition number, computed numerically using an eigenvalue factorization, along with its theoretical lower and upper bounds. Then, we generate random vectors $\v{v}\~[N](\v{0},\m{I})$, compute $\v{u} = \m{K}_{\v{x}\v{x}} \v{v}$, solve for $\v{v} = \m{K}_{\v{x}\v{x}}^{-1} \v{u}$ numerically, and plot the median error norm over $10^4$ samples, along with 25\% and 75\% quantiles. We observe that condition numbers asymptotically grow as $\rho\to1$, but not as $N\to\infty$, and increase hand-in-hand with numerical error.}
\end{figure}

To begin, we first observe that, generically, condition numbers of kernel matrices need not be well-behaved.
To illustrate what can go wrong, consider the one-dimensional exponential kernel
\[
k(x,x') = \exp(-|x-x'|).
\]
Suppose that $x$ is a time series which lies on a regular one-dimensional grid with spacing $r$, so that $x_i = ri$. 
Then the resulting kernel matrix is a \emph{Kac--Murdock--Szegö matrix} \cite{trench2001properties,dow02}, meaning it takes the form
\[
  \m{K}_{\v{x}\v{x}} =  \begin{pmatrix}
    1 & \rho &\rho^2 &\dots & \rho^{n-1} \\
    \rho & 1 & \rho & \dots & \rho^{n-2} \\
    \vdots & \vdots &\ddots & \ddots  & \vdots \\
    \rho^{n-1} & \rho^{n-2} & \rho^{n-3} & \dots &  1
\end{pmatrix}
\]
for $\rho = \exp(-r)$.
To compute a posterior Gaussian process under exact observations, the condition number of the matrix we need to invert to compute the posterior satisfies \cite[p. 9]{trench2001properties} the inequality
\[
\frac{(1+\rho)^2}{(1-\rho)^2} \leq \f{cond}(\m{K}_{\v{x}\v{x}}) \leq \frac{(1+\rho)^2 + 2 \rho \epsilon}{(1-\rho)^2 - 2 \rho \epsilon}
\]
where $\eps = \frac{\pi^2}{(N+1)^2}$ and for the upper bound we have assumed $(1-\rho)^2 > 2 \rho \epsilon$. 
For large $N$, the value $\eps$ is small so the upper and lower bounds essentially match.
As $r \-> 0$, we have $\rho \-> 1$, and the condition number diverges to infinity for any given value of $N$.
On the other hand, if the spread of the data, $r$ is fixed, then regardless of how much data is collected, $\rho$ and the condition number of $\m{K}_{\v{x}\v{x}}$ both remain bounded.
In cases where data is observed under a Gaussian likelihood, one instead needs to invert $\m{K}_{\v{x}\v{x}} + \m\Sigma$.
If $\m\Sigma$ is diagonal, the minimum eigenvalue of $\m{K}_{\v{x}\v{x}} + \m\Sigma$ is bounded below by the minimum diagonal entry of $\m\Sigma$, which we will refer to as $\sigma^2$ by analogy to the homoscedastic noise case when $\m\Sigma = \sigma^2 \m{I}$.
Unfortunately, this is not enough to ensure stability: if one samples the time series more densely, then as $r \-> 0$ and $N \->\infty$, the condition number diverges to infinity for all values of $\sigma^2$ due to the growth of the largest eigenvalue of $\m{K}_{\v{x}\v{x}}$.
We now observe that the behavior illustrated by this example is neither specific to the exponential kernel, nor to data with algebraic structure arising from grids, nor limited to one dimension.

\begin{restatable}{proposition}{PropCovOpLimit}
\label{prop:cov-op-limit}
Let $k$ be a continuous stationary kernel, and let the entries in $\v{x}$ be independently sampled from a uniform distribution on some compact subset $X \subseteq \R^d$.
Define the covariance operator
\[
\c{K} : L^2(X) &\-> L^2(X)    
&
\c{K} : \phi &\|> \int_X \phi(x) k(x,\.) \d x
\]
acting on the Hilbert space of (equivalence classes of) square-integrable functions. 
Then the eigenvalues of $\c{K}$ are countable, non-negative, admit a maximum, and can be ordered to form a non-increasing sequence whose limit is zero. 
For every $n >0$, as $N\->\infty$, the $n$\textsuperscript{th} largest eigenvalue of  $\frac{1}{N}\m{K}_{\v{x}\v{x}}$ converges almost surely to the corresponding eigenvalue of $\c{K}$. 
As a consequence, $\f{cond}(\m{K}_{\v{x}\v{x}}) \-> \infty$ almost surely.
\end{restatable}

\begin{proof}
This follows from \textcite[Corollary 3.3]{koltchinskii2000}, after a bit of technical reformulation and some calculations, which we perform in \Cref{apdx:eig-theory}.
\end{proof}

\Cref{prop:cov-op-limit} implies that $\lambda_{\max}$ grows linearly with $N$, while Cauchy's Interlacing Theorem together with an elementary argument shows that $\lambda_{\min} \-> 0$, causing the condition number of $\m{K}_{\v{x}\v{x}}$ to grow without bound.
From the one-dimensional example and \Cref{prop:cov-op-limit}, we therefore see that to control the condition numbers of kernel matrices, we need to ensure they are sufficiently far away from this ill-conditioned limiting case.
We conclude that if the data is sampled too closely relative to the kernel's length scale, computing a Cholesky factorization of $\m{K}_{\v{x}\v{x}} + \m\Sigma$ numerically is not possible for large matrices.
We proceed to explore this in more general cases in the sequel.

\section{Numerical Stability in Scalable Gaussian Process Approximations}

To build towards Gaussian process approximations which are numerically stable, we study what properties of models and data lead to numerical stability of kernel matrices.
We focus in particular on how their condition numbers depend on the number of input points under general regimes.

\subsection{Numerical Stability in Sparse Approximations via Minimum Separation}

In both exact and sparse Gaussian processes, the linear systems that need to be solved arise from kernel matrices $\m{K}_{\v{x}\v{x}}$ and $\m{K}_{\v{z}\v{z}}$, which depend on the prior and data.
In sparse Gaussian processes, however, $\v{z}$ is not fixed: it is a variational parameter, and is usually selected to maximize approximation accuracy.
We now study whether we can use this freedom to select $\v{z}$ in a way that ensures stability, while maintaining performance.
To this end, we introduce the following notion.

\begin{definition}
\label{def:separation}
Let $\v{z}$ be the inducing points.
Define the \emph{separation radius} 
\[
\f{sep}(\v{z}) = \min_{i \neq j} \norm{z_i - z_j}
.
\]
\end{definition}

We now show that separation is closely connected with numerical stability.
For this, we need a mild regularity condition on the kernel.

\begin{restatable}[Spatial decay]{assumption}{AsmSpatialDecay}
Let $k$ be a kernel on $X \subseteq \R^d$.
We say that $k$ has \emph{spatial decay} if there is a decreasing function $\psi: [0, \infty) \-> [0, \infty)$ such that for all $x,x'\in X$ we have
\[
|k(x,x')| &\leq \psi(\norm{x - x'})
&
\psi(m) &= \c{O}\del{\frac{1}{m^d \log(m)^2 }}
.
\]
\end{restatable}

In words, spatial decay requires that if we consider points far apart, the covariance between the function at those points must tend to zero sufficiently quickly with the distance between them. 
This assumption is satisfied for a very large set of kernels, including potentially non-stationary kernels, as well as the squared exponential kernel and all Matérn kernels, whose decay is controlled above and below by products of an exponential with polynomials.
We now state the main results.

\begin{proposition}
\label{prop:eigenvalue-upper-bound}
Let $X \subseteq \R^d$, and let $k$ satisfy spatial decay.
Then there is a constant $C_{\max}^{k,\delta}$ such that for any $M$ and any $\v{z}$ of size $M$ with $\f{sep}(\v{z}) \geq  \delta > 0$, we have
\[
\lambda_{\max}(\m{K}_{\v{z}\v{z}}) \leq C_{\max}^{k,\delta}
.
\]
\end{proposition}

\begin{proof}
\Cref{apdx:eig-theory}, \Cref{prop:eigenvalue-upper-bound-apdx}.
\end{proof}

From a technical perspective, the claim is proven by combining Gershgorin's Circle Theorem with a packing argument.
Arguments of this kind have a long history in the interpolation literature---see for instance \textcite{narcowich92, narcowich1994condition, schaback95, diederichs2019improved}.
In cases where the Gaussian process approximation requires inversion of $\m{K}_{\v{z}\v{z}} + \m\Lambda$, from this we can immediately conclude a condition number bound.

\begin{corollary}
\label{cor:condition-number-bound-with-noise}
Under the conditions of \Cref{prop:eigenvalue-upper-bound}, for diagonal $\m\Lambda$, letting $\Lambda_{\max}$ and $\Lambda_{\min}$ denote its respective maximum and minimum entry, we have
\[
\f{cond}(\m{K}_{\v{z}\v{z}} + \m\Lambda) \leq \frac{C_{\max}^{k,\delta} + \Lambda_{\max}}{\Lambda_{\min}}
.
\]
\end{corollary}

\begin{proof}
Combine \Cref{prop:eigenvalue-upper-bound} with the triangle inequality for operator norm, noting that for positive semi-definite matrices the operator norm is equal to the largest eigenvalue.
\end{proof}

This generalizes immediately to the case where $\m\Lambda$ is not diagonal, as long as $\Lambda_{\max}$ and $\Lambda_{\min}$ are replaced with its maximum and minimum eigenvalues.
Since in many practical cases $\m\Lambda$ is diagonal, we retain this presentation for simplicity.
In the noiseless case, we can also conclude a condition number bound, but now additionally require stationarity to do so.
We will need the following result.

\begin{result}
\label{res:eigenvalue-lower-bound}
Let $X \subseteq \R^d$, and let $k$ be stationary and continuous.
Then there is a constant $C_{\min}^{k,\delta}>0$ such that for any $M$ and any $\v{z}$ of size $M$ with $\f{sep}(\v{z}) \geq \delta > 0$, we have
\[
\lambda_{\min}(\m{K}_{\v{z}\v{z}}) \geq C_{\min}^{k,\delta}.
\]
\end{result}

\begin{proof}
\textcite[Theorem 12.3]{wendland04}.
See \Cref{apdx:eig-theory} for a sketch of the main argument.
\end{proof}

We can combine this the eigenvalue upper bound to bound the condition number, which we present for completeness.

\begin{corollary}
\label{cor:condition-number-bound}
Under the conditions of \Cref{prop:eigenvalue-upper-bound} and \Cref{res:eigenvalue-lower-bound}, there is a constant $C_{\f{cond}}^{k,\delta}>1$ such that for any $M$ and any $\v{z}$ of size $M$ with $\f{sep}(\v{z}) < \delta$, we have
\[
\f{cond}(\m{K}_{\v{z}\v{z}}) \leq C_{\f{cond}}^{k,\delta}.
\]
\end{corollary}

\begin{proof}
Combine \Cref{prop:eigenvalue-upper-bound} and \Cref{res:eigenvalue-lower-bound} with $\f{cond}(\m{K}_{\v{z}\v{z}}) = \frac{\lambda_{\max}(\m{K}_{\v{z}\v{z}})}{\lambda_{\min}(\m{K}_{\v{z}\v{z}})}$, taking $C_{\f{cond}}^{k,\delta} = \frac{C_{\max}^{k,\delta}}{C_{\min}^{k,\delta}}$.
\end{proof}

These results therefore reveal what properties the set of inducing points $\v{z}$ needs to have in order to yield numerically stable linear systems. 
Separation distance and spatial decay rule out the bad behavior shown in \Cref{prop:cov-op-limit}, by preventing the limiting covariance operator from controlling properties of the corresponding kernel matrices.
The analysis shown is based on the techniques presented by \textcite[Ch. 12]{wendland04} for proving \Cref{res:eigenvalue-lower-bound}, which originate in the polynomial interpolation literature \cite{narcowich92,schaback95}.
For non-stationary kernels, mirroring \Cref{prop:eigenvalue-upper-bound}, we conjecture that a similar bound to \Cref{res:eigenvalue-lower-bound} for $\lambda_{\min}$ holds for a wide class of kernels, but requires a different argument which avoids the use of Fourier analysis.
To conclude, we observe that for most kernels, if one allows for an arbitrarily large number of inducing points and no noise is added to the kernel matrix, then the minimum separation condition is necessary to ensure bounded condition numbers.

\begin{restatable}{proposition}{PropSeparationNecessary}
Let $z_m \in X$, $m=1,..,\infty$, be a sequence, and let $k$ be a Lipschitz continuous kernel on $X \x X$. Define $\m{K}_{\v{z}\v{z}}^{(M)} \in \mathbb{R}^{M \x M}$ to be the kernel matrix formed by evaluating $k$ at $z_1,..,z_M$. 
If there exists a $C>0$ such that for all $M$, $\f{cond}(\m{K}_{\v{z}\v{z}}^{(M)}) \leq C$, then $z_m$ satisfy minimum separation.
\end{restatable}

\begin{proof}
\Cref{prop:apdx-separation-necessary}.
\end{proof}

\subsection{Inducing Points for Geospatial Data via Cover Trees}

To obtain an algorithmic solution that guarantees numerical stability, we study how to construct inducing points which satisfy a user-specified minimum separation radius.
Since separation is defined purely in terms of the inducing points, we need a second criterion to quantify how well the inducing points summarize a given dataset.
A clear choice is to place the inducing points as close to the data as possible.

\begin{definition}
\label{def:spatial-resolution}
Let $\v{x}$ be the data, and $\v{z}$ be the inducing points.
Define the \emph{spatial resolution}
\[
\f{res}_{\v{x}}(\v{z}) &= \max_{i=1,..,N} \min_{j=1,..,M} \norm{x_i - z_j}
.
\]
\end{definition}

In simple terms, the spatial resolution is the maximum distance from a data point to the closest inducing point.
This notion is closely related to \emph{fill distance} which appears in the the analysis of posterior contraction rates \cite{kanagawa18}, but instead involves both a dataset $\v{x}$ and inducing points $\v{z}$.
Among all sets of inducing points with a given spatial resolution, finding ones which satisfy separation is always possible.
More precisely, if we take the inducing points $\v{z}$ to be the centers of balls which make up a maximal separated subset of radius $\f{sep}(\v{z}) = \eps$, the resulting choice automatically has spatial resolution $\f{res}_{\v{x}}(\v{z}) = \eps$.
A similar observation applies if one takes $\v{z}$ to be an optimal covering of $\v{x}$---see \Cref{apdx:eig-theory} for details.
This shows that the notions of separation distance and spatial resolution are closely connected with optimal covering and packing numbers---see \textcite[Section 4.2]{vershynin18}.

We now study numerical techniques for computing inducing points that simultaneously satisfy separation and covering guarantees. 
Our goal is to obtain a practical near-linear-time method producing a set of inducing points with small spatial resolution, to ensure accurate approximations of the posterior, and large separation to ensure stability. 
To achieve this, we propose an approach we term the \emph{$R$-neighbor cover tree}, which is a refinement of the \emph{cover tree} data structured originally proposed by \textcite{beygelzimer06} for efficient nearest neighbor search, and studied further by \textcite{kibriya07,izbicki15}.
 This approach is well-suited to cases where the inputs are low-dimensional such as geospatial data. We note that the proposed algorithm will also perform well in cases when the dataset has low \emph{intrinsic dimension}---see \textcite{beygelzimer06}---even if the dataset is embedded in a high dimensional space. 

\begin{figure}
\begin{subfigure}{0.26\textwidth}
\includegraphics{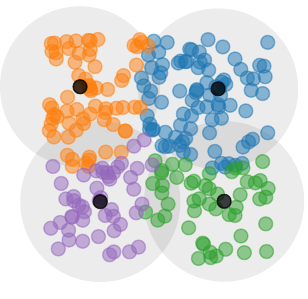}
\caption{Place covering nodes}
\end{subfigure}
\begin{subfigure}{0.26\textwidth}
\includegraphics{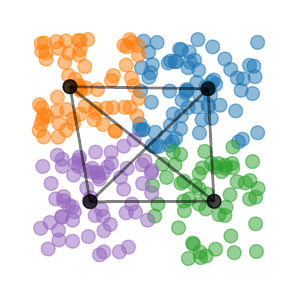}
\caption{Compute $R$-neighbors}
\end{subfigure}
\begin{subfigure}{0.44\textwidth}
\hfill
\includegraphics{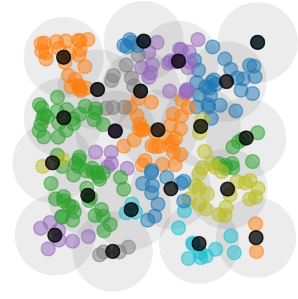}
\hfill
\hfill
\includegraphics{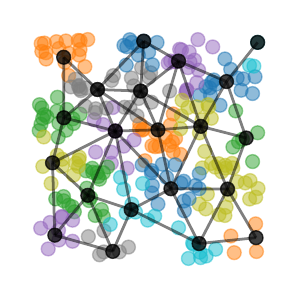}
\hfill
\caption{Repeat steps (a) and (b) recursively}
\end{subfigure}
\caption{Two iterations of \Cref{alg:covertree} on a tree with $L=3$ total levels. Given a region, the algorithm first (a) computes a covering of the region by picking data points one-by-one from those not yet covered. Then, it (b) computes the $R$-neighbors of each node and uses this to efficiently compute a Voronoi tessellation of the region. Finally, the algorithm (c) repeats the process recursively.}
\label{fig:covertree}
\end{figure}

A cover tree is a tree where nodes are points associated with metric balls.
The root node is a ball which covers all of the data $\v{x}$.
Each parent node contains the centers of its child nodes, which consist of balls of smaller radius than the parent.
At each level, the tree enforces separation and covering properties similar to those in \Cref{def:separation} and \Cref{def:spatial-resolution}, with the separation distance and spatial resolution proportional to the radius of balls at that depth.
\textcite{beygelzimer06} present a depth-first algorithm for building such a tree.

Since we are interested in inducing point selection with a fixed spatial resolution $\eps$, rather than nearest neighbors search, we seek a different set of guarantees than those provided by \textcite{beygelzimer06}, or by improved versions such as that of \textcite{izbicki15}.
The most important difference is that, in the inducing point case, we want separation to hold \emph{globally} in every level of the tree, and not only for children with a common parent.
To do this, we present a modified breadth-first construction, which is of independent interest, and conceptually works as follows.

\1 Initialize the root node $z_{1,1}$ at the mean of the data, and assign all data to it.
\2 Loop over tree depth $\ell$, starting at $\ell=0$.
\1 Loop over parent nodes $z_{\ell,p}$.
\1 Select a point $z'$ from the parent's assigned data.
\2 \emph{Optional:} compute the local average of the parent node's assigned data around $z'$, and set $z'$ to be this average, as long as it is not too close to another node.
\3 Create a child node $z_{\ell+1,c}$ centered at $z'$.
\4 Reassign all points near $z_{\ell+1,c}$ in nodes at level $\ell$ to $z_{\ell+1,c}$. 
\0
\5 \emph{Optional:} reassign all points at level $\ell+1$ according to the Voronoi partition of $z_{\ell+1,c}$.
\0
\0 

A formal description is given in \Cref{alg:covertree}.
The key insight behind this top-down, breadth-first construction is that, for a dataset with a given intrinsic dimensionality, it can be performed in near-linear time.
The reason for this is that every step is \emph{local}, and can be computed by searching only the children of nodes sufficiently close to the current parent node.
More precisely, we show that if the parent node's radius is at level $\ell$ with radius $R$ and the total depth of the tree is $L$, then only nodes within a distance of $4(1 - 1/2^{(L-\ell)})R \leq 4R$ need to be searched.
The number of such nodes is bounded, can be tracked recursively, and is controlled by the intrinsic dimensionality of the data.

The two optional steps do not alter the algorithm's complexity or its guarantees, but improve practical performance.
The local averaging step we propose is similar to the Lloyd's iteration of $K$-means, and allows the tree to use fewer nodes by placing them in-between data points, rather than exactly on top of data points.
The Voronoi repartitioning step, originally proposed by \textcite{izbicki15}, makes the tree better-balanced.
We now prove the algorithm works as intended.

\begin{algorithm}[t]
\caption{Cover Tree Inducing Points}
\label{alg:covertree}
\begin{algorithmic}[1]
\csname ALC@it\endcsname\textbf{input}{}\ \  spatial resolution $\eps>0$ and dataset $\v{x}$. Define notation $B_z(R) = \{x\in X : \norm{z - x} \leq R\}$.\ignorespaces
\STATE Initialize root node $z_{0,1} = \frac{1}{N} \sum_{i=1}^N x_i$, assigned data $\c{A}_{1,1} = \v{x}$, $R$-neighbors $\c{R}_{1,1} = \{1\}$, as well as constants $d_{\max} = \max_{i=1,..,N} \norm{z_{1,1} - x_i}$, $L = \left\lceil\log_2\frac{d_{\max}}{\eps}\right\rceil$, $M_0 = 1$, and $R_0 = 2^{L} \eps$.
\FOR{tree depth level $\ell = 1,..,L$}
\STATE Initialize number of nodes $M_{\ell} = 0$, child node index $c = 1$, and radius $R_{\ell} = \frac{1}{2} R_{\ell-1}$.
\FOR{parent node index $p = 1,..,M_{\ell-1}$}\label{line:parent-loop}
\STATE Initialize child nodes $\c{C}_{\ell-1,p} = \emptyset$ and, \emph{optionally}, assigned data copy $\c{A}'_{\ell-1,p} = \c{A}_{\ell-1,p}$.
\WHILE{assigned data $\c{A}_{\ell-1,p} \neq \emptyset$}
\STATE Choose an arbitrary data point $\zeta \in \c{A}_{\ell-1,p}$. \label{line:choose-new-zeta}
\STATE \emph{Optional:} compute local average $\zeta' = \frac{1}{|\c{Z}|} \sum_{z \in \c{Z}} z$, where $\c{Z} = \c{A}_{\ell-1,p} \^ B_\zeta(R_\ell)$.
\STATE\quad\textbf{if} $\min_{z \in \c{Z}'} \norm{\zeta - z} > R_\ell$ where $\c{Z}' = \U_{r \in \c{R}_{\ell-1,p_{\ell,c}}} \c{C}_{\ell-1,p}$ \textbf{then} set $\zeta = \zeta'$. \label{line:zeta-average}
\STATE Create child node $z_{\ell,c} = \zeta$.
\FOR{parent $R$-neighbor index $r$ in $\c{R}_{\ell-1,p}$}
\STATE Update assigned data $\c{A}_{\ell,c} = \c{A}_{\ell,c} \u (\c{A}_{\ell-1,r} \^ B_{z_{\ell,c}}(R_{\ell}))$, $\c{A}_{\ell-1,r} = \c{A}_{\ell-1,r}\takeaway B_{z_{\ell,c}}(R_{\ell})$. \label{line:update-data} 
\ENDFOR
\STATE Update children $\c{C}_{\ell-1,p} = \c{C}_{\ell-1,p} \u \{c\}$, parent $p_{\ell,c} = p$, index $c = c + 1$, and $M_{\ell} = M_{\ell} + 1$.
\ENDWHILE
\ENDFOR
\FOR{child node index $\varsigma = 1,..,c$}
\STATE Compute child $R$-neighbors $\c{R}_{\ell,\varsigma} = \U_{r \in \c{R}_{\ell-1,p_{\ell,c}}} \c{C}_{\ell-1,r} \^ B_{z_{\ell,c}}(4(1 - 1/2^{(L-\ell)})R_{\ell})$. \label{line:neighbor-defn}
\STATE \emph{Optional:} recompute $\c{A}_{\ell,\varsigma} = \{x \in \U_{r\in \c{R}_{\ell-1,p_{\ell-1,c}}} \c{A}'_{\ell,r} : \norm{x - z_\varsigma} \leq \norm{x - z_{\varsigma'}}, \forall\varsigma'\neq\varsigma\}$.
\ENDFOR
\ENDFOR
\RETURN $\v{z} = \{z_{\ell,m} : \ell = L\}$.
\end{algorithmic}
\end{algorithm}

\begin{restatable}{theorem}{ThmCoverTree}
\label{thm:cover-tree}
For a given target spatial resolution $\eps > 0$, and dataset with $\bar{x} = \frac{1}{N}\sum_{i=1}^N x_i$ and $d_{\max} = \max_i \norm{x_i - \bar{x}}$, the cover tree inducing point algorithm, with or without the optional steps, produces a tree with $L = \left\lceil\log_2\frac{d_{\max}}{\eps}\right\rceil + 1$ levels, terminates in 
\[
\c{O}\del{ P_{B(4)}^{\f{ext},1/2} \del{P_{B(1)}^{\f{ext},1/2}}^2 N \log\frac{d_{\max}}{\eps}}
\] 
steps, where $B(r)$ denotes a ball of radius $r$, and $P_X^{\f{ext},\delta}$ is the external packing number, namely the maximum number of disjoint balls of radius $\delta$ that one can choose so that their centers lie in $X$.
Moreover, the algorithm guarantees that the nodes at each level $\ell$ satisfy 
\[
\f{res}_{\v{x}}(\v{z}_\ell) &\leq 2^{L-\ell} \eps
&  
\f{sep}(\v{z}_\ell) &\geq 2^{L-\ell} \eps
.
\]
\end{restatable}

\begin{proof}
\Cref{apdx:cover-tree}.
\end{proof}

This gives an efficient and practical algorithm for choosing inducing points which guarantees the properties needed in \Cref{cor:condition-number-bound} hold.
The algorithm's output is fairly-sharp: compared to an optimal covering, the cover tree construction may require a slightly larger number of inducing points, but achieves the same separation guarantee, irrespective of the number of levels.
Using this, we obtain the following numerical stability guarantee.

\begin{corollary}
Let $\v{x}$ be a dataset of size $N$, and let $\eps > 0$
Let $\v{z}$ be a set of inducing points computed from $\v{x}$ by the cover tree algorithm with target spatial resolution $\eps$.
Let $k$ be stationary, continuous, and satisfy spatial decay.
Then $\f{cond}(\m{K}_{\v{z}\v{z}})$ is bounded independently of $N$.
\end{corollary}

\begin{proof}
Combine \Cref{thm:cover-tree} with \Cref{cor:condition-number-bound}.
\end{proof}

For kernels which possess a length scale parameter, the desired spatial resolution in a given problem might depend on the value of the length scale.
This parameter is often set by training the model via maximum marginal likelihood, and can therefore change as training progresses.
In such settings, one can use the hierarchical structure of the tree to dynamically adjust the number of inducing points as the length scale changes.
This can be done by switching to inducing points obtained from an intermediate level of the tree if the length scale becomes coarser, or by computing an extra level of the tree if it becomes finer.
The simplest way to do this is to adjust the model after initial training is complete, since this avoids non-smooth changes in the inducing point approximation during training, and fine-tune the model after the update if necessary.

\subsection{Additional Stability via the Clustered-data Inducing Point Approximation}
\label{sec:cdgp}

In the preceding sections, we studied properties of data needed to obtain a stability guarantee for inducing point methods.
We now study a complementary question: if selecting $\v{z}$ using a cover tree, can the inducing point approximation be modified to improve stability further?
To proceed, we begin by defining a low-rank analog \cite{bui2017unifying,panos18,adam2021dual} of the \textcite{opper09} approximation
\[
\label{eqn:var-approx-pathwise}
(f\given \v{u})(\.) &= f(\.) + \m{K}_{(\.)\v{z}} (\m{K}_{\v{z}\v{z}} + \m\Lambda)^{-1}(\v{u} - f(\v{z}) - \v\epsilon)
&
\v\epsilon \~[N](\v{0},\m\Lambda)
\]
where $\v{z}$, diagonal noise matrix $\m\Lambda$, and the inducing mean value $\v{u}$ are now variational parameters.
Since its optimization objective is non-convex in the variational parameters, this approximation is generally used with natural gradients \cite{adam2021dual} or stochastic optimization \cite{panos18,vdw2022inv}.
Similar to the earlier inducing point approximation of \textcite{titsias09,hensman13}, it recovers the true posterior when $M = N$, $\v{z} = \v{x}$, $\m\Lambda = \m\Sigma$, and $\v{u} = \v{y}$.

This approximation is more numerically stable thanks to the diagonal matrix $\m\Lambda$ \cite{panos18,vdw2022inv}, which lower bounds the minimum eigenvalue of $\m{K}_{\v{z}\v{z}} + \m\Lambda$ by $\min_{i=1,..,M} \Lambda_{ii} + \lambda_{\min}(\m{K}_{\v{z}\v{z}}) > 0$.
For finite datasets, this is a strict improvement in stability over approximations involving only $\m{K}_{\v{z}\v{z}}$.
Since $\m\Lambda$ effectively controls the width of the posterior error bars, the degree to which this benefit manifests itself depends on the data, with more uncertainty leading to more numerically stable linear systems.

\begin{figure}
\hfill
\begin{subfigure}{0.3\textwidth}
\includegraphics{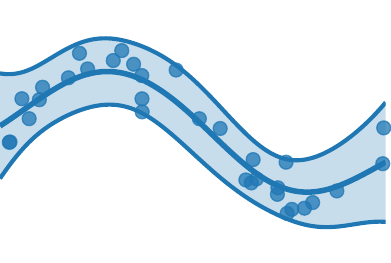}
\caption{Original Gaussian process}
\label{fig:orig-gp}
\end{subfigure}
\hfill
\begin{subfigure}{0.3\textwidth}
\includegraphics{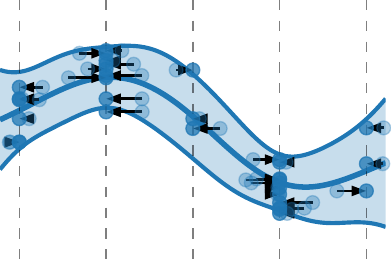}
\caption{Clustered-data approx.}
\label{fig:cluster-gp}
\end{subfigure}
\hfill
\begin{subfigure}{0.3\textwidth}
\includegraphics{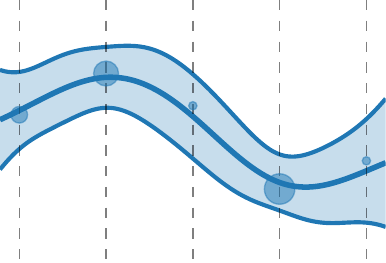}
\caption{Equivalent sparse process}
\label{fig:equiv-sparse-gp}
\end{subfigure}
\hfill
\caption{The idea behind the clustered-data approximation of \Cref{prop:inducing-point-repr}. To construct this approximation, we move each data point in $\v{x}$ to its nearest cluster, transforming (a) to (b). We can then reinterpret and represent (b) in a sparse manner by merging repeated data points, replacing $\v{y}$ with the cluster means $\v{u}$, and re-weighting each data point according to cluster size---obtaining (c), which is equal in distribution to (b).}
\label{fig:inducing-points}
\end{figure}

Before studying convergence further, we observe that this approximation's relationship with the true posterior is deeper than might first meets the eye.
For a given set of inducing points $\v{z}$ and dataset $\v{x}$, define the \emph{nearest inducing point clustering} and \emph{cluster size} maps
\[
\f{cl}(x) &= \argmin_{z_j,\,j=1,..,M} \norm{x - z_j}
&
N_{\f{cl}}(x) &= |\{x' \in \v{x} : \f{cl}(x) = \f{cl}(x')\}|
.
\]
With these notions, we show that, under certain choices of the variational parameters, the Opper--Archambeau approximation is exactly the correct posterior under a perturbed dataset, where $\v{x}$ is replaced with $\f{cl}(\v{x})$ and all other data and model parameters are unchanged.

\begin{restatable}{proposition}{PropInducingPointRepr}
\label{prop:inducing-point-repr}
Let $u_j = \frac{1}{N_{\f{cl}}(z_j)}\sum_{\f{cl}(x_i) = z_j} y_i$.
The Bayesian models defined by $f \~[GP](0,k)$ and
\[
y_i \given f &\~[N](f(\f{cl}(x_i)),\sigma^2)
&
u_i \given f &\~[N]\del{f(z_i), \frac{\sigma^2}{N_{\f{cl}}(z_i)}}
.
\]
admit respective posterior distributions $f\given\v{y}$ and $f\given\v{u}$ which are equal in distribution.
\end{restatable}

\begin{proof}
\Cref{apdx:inducing-points}.
\end{proof}

In cases where the spatial resolution of \Cref{def:spatial-resolution} is small---which the cover tree guarantees---we can use \Cref{prop:inducing-point-repr} to accurately estimate the correct variational parameters without needing to perform optimization.
This is done by choosing $\v{z}$ to be the clusters, choosing $\v{u}$ to be the cluster mean, and choosing the entries of the heteroskedastic noise matrix $\m\Lambda$ to be equal to the original noise variance divided by the cluster size.
We call this the \emph{clustered data inducing point approximation}, and illustrate it in \Cref{fig:inducing-points}.

The explicit description of the approximation made by use of inducing points in \Cref{prop:inducing-point-repr} enables us to better understand how to select them algorithmically.
To develop this line of thought in more detail, we proceed to study how to choose inducing points to optimize the tradeoff between approximation error and computational costs using conjugate gradients.

In many applications, employing this approximation involves computing its marginal likelihood objective, or related quantities such as prior density evaluations or Kullback--Leibler divergences.
One of the noteworthy properties of this approximation is that the Kullback--Leibler divergence between it and the true prior simplifies into
\[
D_{\f{KL}}(q \from p) = \frac{1}{2}\ln\frac{|\m{K}_{\v{z}\v{z}} + \m\Lambda|}{|\m\Lambda|} - \frac{1}{2}\tr((\m{K}_{\v{z}\v{z}} + \m\Lambda)^{-1}\m{K}_{\v{z}\v{z}}) + \frac{1}{2}\v{v}^T\m{K}_{\v{z}\v{z}}\v{v}
\]
where $\v{v} = (\m{K}_{\v{z}\v{z}} + \m\Lambda)^{-1}\v{u}$, and the minus sign in the trace term appears due to cancellation with the usual dimension term.
The linear systems appearing in this term are more numerically stable than those in most other variational approximations, owing to the presence of $\m\Lambda$.
This applies to all computations needed to train the Gaussian process, not just the ones for obtaining predictive means and variances.

\section{Experiments}

We investigate the performance of the clustered-data inducing point approximation on a series of tests designed to illustrate both the behavior of its components and the overall picture.
Each experiment focuses on a different aspect of the method, including approximation error, scalability, and behavior under varying quantities of data and inducing points.
Full details regarding the experimental setup are given in \Cref{apdx:experiments}.

\subsection{Spatial Resolution and Empirical Approximation Error}

To better understand how different spatial resolutions resolve the tradeoff between approximation error and performance, we performed a number of numerical experiments designed to highlight this tradeoff.
To quantify approximation error, we chose the Wasserstein distance, because it is finite in the setting of interest, is readily estimatable numerically, and controls expectations of a large class of posterior functionals \cite{villani08}.

To assess these quantities, we generated synthetic data by sampling $N=1000$ input points uniformly on a hypercube of dimension $1,2,4$ and $8$, and then sampling from the prior process, which uses a squared exponential kernel.
We then varied the spatial resolution, computed inducing points using the cover tree algorithm, and empirically estimated the Wasserstein distance between the approximate and true Gaussian processes.
To better contextualize results, since the cover tree construction produces a variable number of inducing points, we also recorded the number of inducing points produced for spaces of different dimension.

Results can be seen in \Cref{fig:spatial-res-results}.
Immediately, we see that behavior of the cover tree depends on the dimension of the space the data lives on.
In higher dimension, more inducing points are needed to construct a separated covering with the same spatial resolution.
On the other hand, we see spatial resolution directly controls approximation accuracy in Wasserstein distance, with finer spatial resolutions more accurate than coarser ones, and the effect being similar in different dimensions.

Next, to assess the consequence of different spatial resolutions on stability and corresponding difficulty of solving the resulting linear systems, we assess empirical performance of conjugate gradients in the settings of interest.
We adopt the same data-generation process used in the preceding variant, but this time focus on the linear systems.
We again vary spatial resolution, but now compute the numerical condition numbers of the resulting kernel matrices, and run the conjugate gradient algorithm to convergence, obtaining the number of iterations needed for convergence.

In \Cref{fig:spatial-res-results}, we see that finer spatial resolutions generally result in more expensive-to-solve and less-stable linear systems.
The effect varies according to dimension, and is least pronounced in $d=8$: in high dimensions, when sampling uniformly on a hypercube, the inputs tend to end up far away from each other, leading to less extreme eigenvalue s.
In lower dimensions, where the posterior is more concentrated, the effect of spatial resolution on condition numbers is more pronounced.
In all cases, higher condition numbers lead to more iterations conjugate gradient iterations needed.
In total, we see that spatial resolution is a directly interpretable quantity which facilitates the tradeoff between approximation accuracy with stability and computational cost.

\begin{figure}
\includegraphics{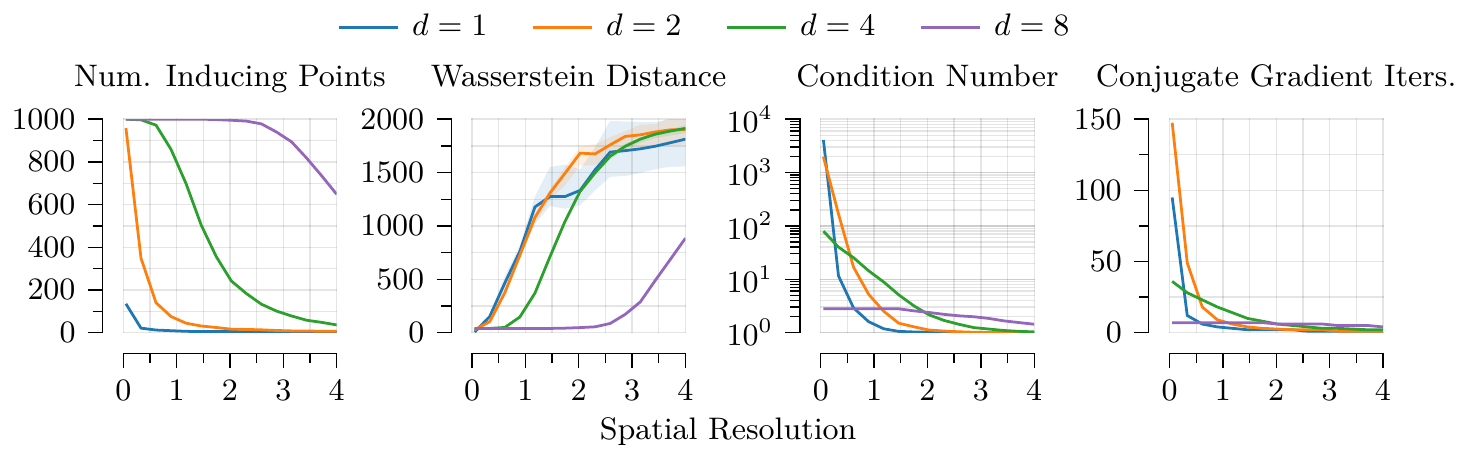}
\caption{We illustrate how spatial resolution affects approximation accuracy and computational cost. We see that increasing the spatial resolution leads to using fewer inducing points, dropping rapidly in low dimension, and slower in higher dimensions. On the other hand, the error in approximation, as measured by Wasserstein distance between the approximate posterior and exact posterior, increases as the spatial resolution increases. Computational costs follow the opposite relationship: finer spatial resolutions are less stable and more expensive to compute, as evidenced by condition numbers and the number of iterations needed for conjugate gradients to converge.}
\label{fig:spatial-res-results}
\end{figure}

\subsection{Comparison to Alternative Inducing Point Selection Schemes}

Next, to understand how different inducing point selection methods perform, we performed an experiment to assess the tradeoffs in terms of performance, stability, and computational cost.
We compare five methods for selecting inducing points: the cover tree approach described in \Cref{alg:covertree}, the online inducing point selection algorithm proposed by \textcite{galyfajou21}, the partial pivoted Cholesky approach studied by \textcite{foster09}, along with baselines consisting of ordinary $K$-means, $K$-means++ \cite{arthur2006k}, optimizing the inducing points directly by minimizing the respective Kullback--Leibler divergence, and uniform random sampling from the training data.

To assess performance and stability, we evaluated the methods in terms of their produced kernel matrix condition numbers, and three sparse Gaussian process regression performance metrics, namely root mean squared error, negative log predictive density, and the evidence lower bound. 
Together, these give a view of both training and test-set performance, as well as variational approximation error.
We chose the sparse approximation of \textcite{titsias09} owing to its widespread use, working with the \emph{GPFlow}, whose linear algebra is implemented in a stable manner using the $V$-method suggested by \textcite{foster09}.
We use the squared exponential kernel and a jitter value of $10^{-6}$.
We evaluated the different methods on two datasets: the two-dimensional \emph{East Africa} geospatial dataset studied by \textcite{wan02,weiss14,ton2018}, and the four-dimensional \emph{Combined Cycle Power Plant} dataset from the UCI Machine Learning Repository \cite{tufekci2014power, dua2019uci}.
These give a view of performance in both geospatial and non-geospatial settings, while ensuring all methods run successfully and can therefore be meaningfully compared.

\begin{figure}[t]
\includegraphics{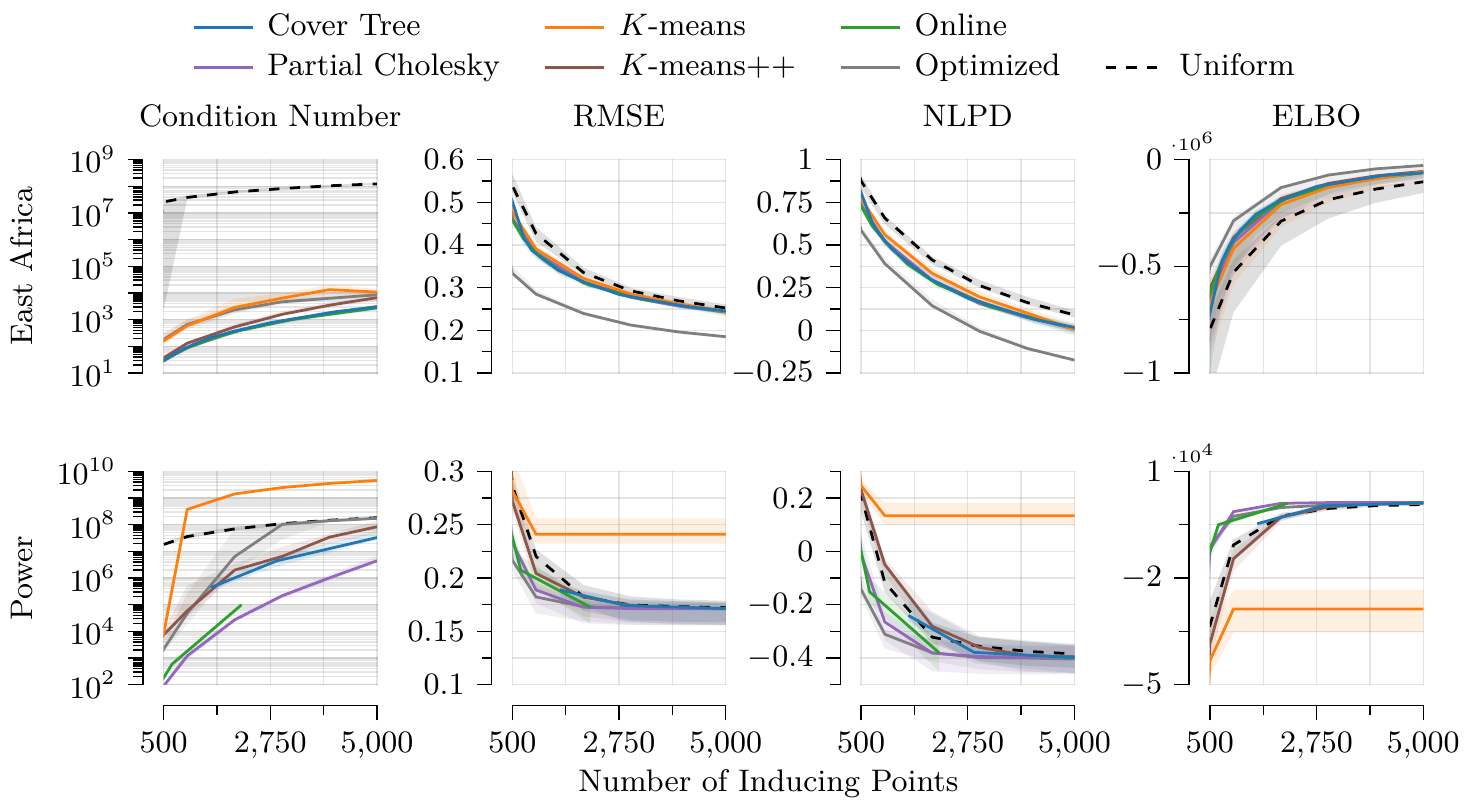}
\caption{We compare how different inducing point selection algorithms vary according to performance and stability, using the geospatial \emph{East Africa} dataset, and higher-dimensional \emph{Power} dataset. The number of inducing points produced by algorithms whose output is variable is shown rounded to the nearest increment. Note that the computational complexity of these methods differs: cover tree is $\c{O}(N\log N)$, online and the $K$-means variants are $\c{O}(NM)$, partial pivoted Cholesky is $\c{O}(NM^2)$, and optimizing the inducing points is $\c{O}(NM^2)$ per optimization iteration. We see that using more inducing points results in higher performance, but reduced stability, with the precise effect varying according to dataset, dimensionality, and inducing point selection method.}
\label{fig:inducing-comp-results}
\end{figure}

Results can be seen in \Cref{fig:inducing-comp-results}.
Immediately, we see that different methods produce different tradeoffs between performance and stability.
These tradeoffs need not be achieved in an optimal manner: one might intuitively expect the ordinary $K$-means baseline to produce inducing points that are well-separated due to its clustering nature, and therefore result in stable linear systems.
This doesn't happen: instead, $K$-means tends to simultaneously achieve the worst performance and worst numerical stability among all methods except possibly uniform subsampling.
In comparison, $K$-means++, which uses a sophisticated initialization scheme that promotes separation, produces much better performance and stability than $K$-means, though still not as well as methods with specified separation guarantees.
This shows that one must explicitly encourage separation to obtain good performance-stability tradeoffs via such approaches.

The strongest predictive performance is achieved by optimizing the inducing points directly by minimizing the respective Kullback--Leibler divergence using gradient descent.
The resulting inducing points, however, can be as numerically unstable as those produced by $K$-means or uniform subsampling.
Since optimizing inducing point locations requires one to specify an initialization, this suggests it may be advantageous to choose this initialization carefully, using one of the stable inducing point selection algorithms, to minimize potential numerical failures during the optimization process.

On the geospatial example, again somewhat-surprisingly, all methods with stability guarantees produced near-identical condition numbers and performance for a given number of inducing points.
The main difference with these methods is therefore their computational cost, which in low dimension are $\c{O}(N\log N)$ for the cover tree, $\c{O}(NM)$ for the online approach and $K$-means++, and $\c{O}(NM^2)$ for partial pivoted Cholesky.
The online method's complexity can, for some kernels, be improved to $\c{O}(N\log N)$ by for instance precomputing a tree-based data structure that enables fast nearest-neighbor lookups, in which case the resulting variant starts to resemble the final loop of the cover tree construction algorithm.
Overall, we conclude, in this setting, that significant numerical stability improvements can be obtained without reducing performance---and, using the cover tree, without paying excessive computational costs.

On the higher-dimensional non-geospatial example, no method of inducing point selection achieves a clear improvement on top of uniform subsampling of the data in terms of performance.
On the other hand, the cover tree, online method, $K$-means++, and partial pivoted Cholesky all achieve better numerical stability compared to $K$-means, uniform sampling, and optimization, with partial pivoted Cholesky being the most stable.
Due to the curse of dimensionality, the $\c{O}(NM^2)$ computational costs of partial pivoted Cholesky are similar or favorable to the other approaches, making this method competitive.
In this setting, we therefore conclude that significant numerical stability improvements can be obtained---however, unlike in the lower-dimensional setting, the computational costs needed to do so can become expensive.

In summary, we conclude from the comparisons that different approaches' strengths are most effective in different settings.
In low-dimensional geospatial problems, among all non-iterative methods, the cover tree algorithm provides good performance with guaranteed stability, while running in near-linear time.
The online approach performs comparably, and can be more appropriate for settings where data arrives sequentially.
$K$-means++, which promotes but does not enforce separation, performs slightly worse.
In high-dimensional problems, the partial pivoted Cholesky algorithm results in strong performance, provided the desired number of inducing points is small enough that running it is practical.
One can obtain the best performance by optimizing the inducing points directly, at cost of having to run an iterative procedure, and producing worse stability.
In both classes of problems, our experiments suggest that methods with stability guarantees can offer a better performance-stability tradeoff compared to naïve methods.

\subsection{Reliability in Floating Point Precision: Geospatial Illustrative Example}

\begin{figure}
\begin{subfigure}{0.24\textwidth}
\includegraphics[scale=0.25]{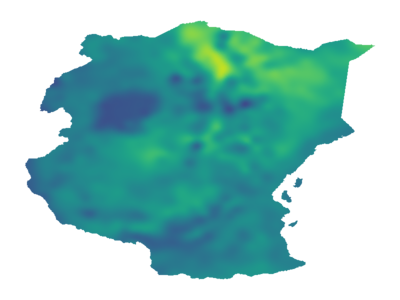}
\\
$\eps = 0.09$
\\
$M = 902$
\end{subfigure}
\begin{subfigure}{0.24\textwidth}
\includegraphics[scale=0.25]{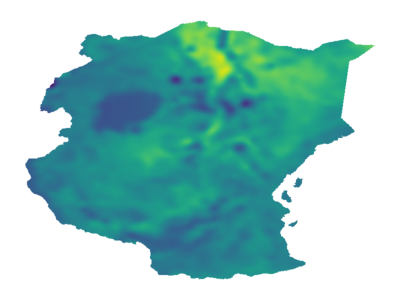}
\\
$\eps = 0.06$
\\
$M = 1934$
\end{subfigure}
\begin{subfigure}{0.24\textwidth}
\includegraphics[scale=0.25]{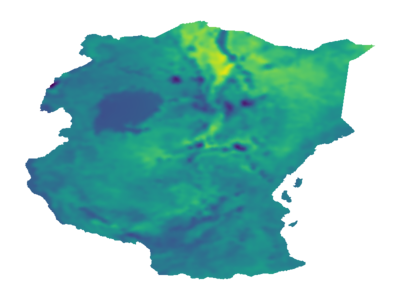}
\\
$\eps = 0.03$
\\
$M = 6851$
\end{subfigure}
\begin{subfigure}{0.25\textwidth}
\includegraphics{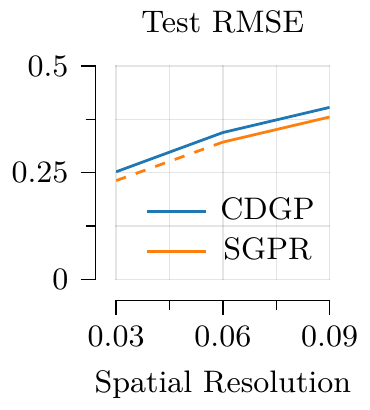}
\end{subfigure}
\caption{Here we illustrate the effect of spatial resolution $\eps$ and resulting number of inducing points $M$ on the stability and performance of the Gaussian process model for the clustered-data approximation (blue), and sparse variational Gaussian process baseline (orange). We see that coarser spatial resolutions result in blurrier predictions and, in turn, higher RMSE. While the sparse variational Gaussian process baseline achieves better performance on coarser resolutions, its Cholesky factorization eventually fails, requiring us to increase jitter---this is displayed by the dashed line. The clustered-data approximation with conjugate gradients, on the other hand, runs reliably in all cases, achieving slightly worse but overall comparable performance.}
\label{fig:geospatial}
\end{figure}

One of the appealing properties of using numerically stable approaches in variational inference is that, in principle, they can enable one to use a larger set of inducing points than customarily used.
This is for two reasons: (i) the quadratic cost of conjugate gradients and (ii) the ability to save memory by effectively running in floating point precision.
To illustrate that using more inducing points can improve performance on applied examples, we train the variational approximation of \textcite{titsias09} and with the clustered-data approximation of \Cref{sec:cdgp} on the East Africa land surface temperature dataset \cite{wan02,weiss14,ton2018}.
In both cases, we use floating-point precision, and the same inducing points obtained using a cover tree.

One of the defining characteristics of this dataset is that it has a relatively small number of data points per unit length scale ball, due to the fine-scale variation in land surface temperature.
This renders approximations based on a sparse set of inducing points less accurate, unless the number of inducing points approaches the order of the size of the dataset. 
By varying the spatial resolution from relatively long ($\eps = 0.09$)  to very short ($\eps = 0.03$) we see both qualitative and quantitative improvement, shown in \Cref{fig:geospatial}.

When using larger sets of inducing points and running in floating point precision, numerical stability becomes increasingly crucial for running successfully.
For the variational approximation of \textcite{titsias09}, this manifests itself as Cholesky factorization failure.
A common heuristic for alleviating this in the Gaussian process literature is to add \emph{jitter} to the kernel matrix: this is done by replacing $\m{K}_{\v{z}\v{z}}$ with $\m{K}_{\v{z}\v{z}} + \eps\m{I}$ for a small value $\eps>0$. 
This increases all of the eigenvalues of the matrix by $\eps$, and ensures that the smallest 
eigenvalue cannot be arbitrarily close to $0$.
We found that using a larger jitter value compared to the \emph{GPflow} recommended default was necessary when working with the finest spatial resolution of $\eps = 0.03$, which produced $M=6851$ inducing points, indicating that this separation distance alone was too small to ensure stability in this example.

In comparison, the clustered-data approximation does not require jitter, since it only involves inversion of the matrices $\m{K}_{\v{z}\v{z}} + \m\Lambda$ which include the diagonal matrix $\m\Lambda$.
This makes it a more reliable approach, but results in slightly worse predictive performance.
In settings where the Gaussian process is trained in an automated manner without human oversight, or where data is collected online and not available in advance, reliability can be a bigger concern than performance.
In such cases, our results indicate that the clustered-data approximation can be an alternative worth considering.

In this example, numerical stability is first and foremost a consequence of the minimum kernel matrix eigenvalue $\lambda_{\min}(\m{K}_{\v{z}\v{z}})$ being too close to zero.
Both introducing jitter and using the clustered-data approximation can alleviate this, but cannot help with issues arising from the maximum eigenvalue being too large, which must eventually occur if many inducing points accumulate in a bounded region.
Our experiment indicates that this does not occur on the scales considered, but might instead occur when working with substantially larger sets of inducing points, for instance in the presence of sparse or otherwise structured kernel matrices \cite{durrande19}, which can often be inverted at substantially larger scale.

\subsection{Numerical Stability and Dataset Size}

\begin{figure}
\includegraphics{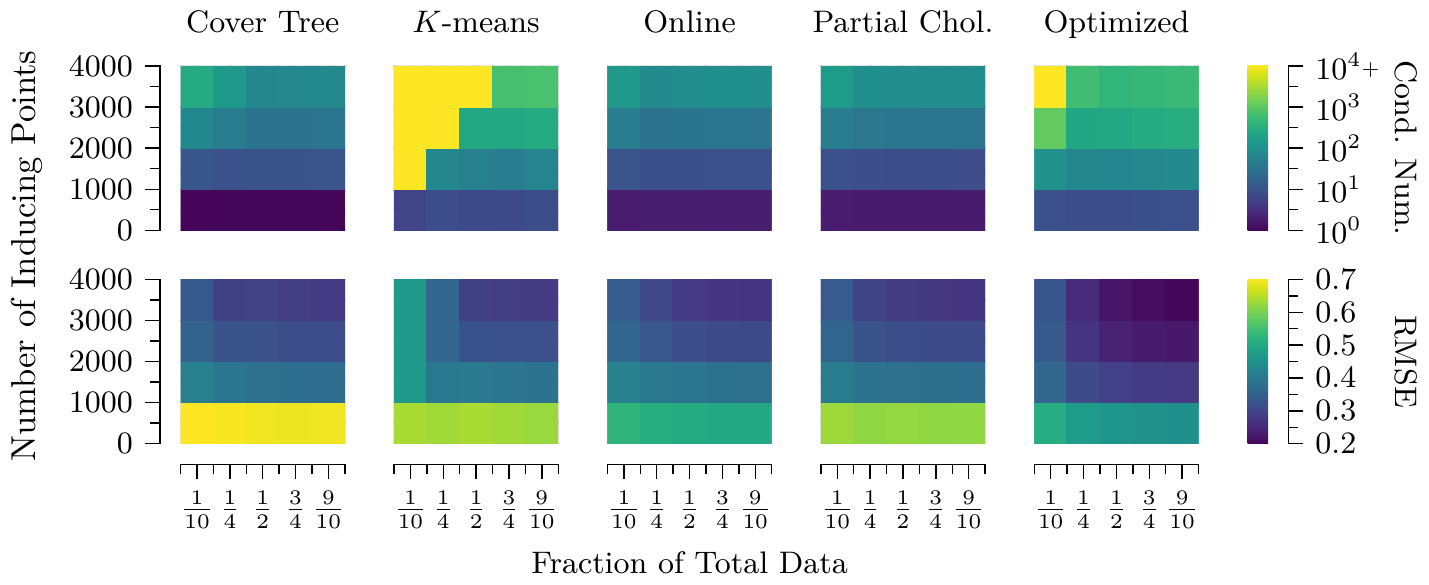}
\caption{We illustrate how numerical stability and performance change when data size and approximation accuracy are varied simultaneously. Viewing a single plot top-to-bottom, we see that for a fixed data size, increasing the number of inducing points produces better performance and very slightly improves stability. Viewing the same plot left-to-right, we see that for a fixed level of variational approximation expressiveness, using more inducing inducing points produces better performance but decreases stability. Viewing the plot diagonally, we see that increasing data size and variational approximation expressiveness simultaneously results in better performance but worse stability, with different methods choosing tradeoffs in favor of either performance or stability, depending on the details of how they are designed.}
\label{fig:data-size}
\end{figure}

The preceding sections show that, in the examples considered, numerical stability tends to deteriorate as approximation accuracy increases, as measured by the number of inducing points.
On the other hand, in many settings, the more data one has, the more inducing points one needs to accurately represent the true posterior.
One can therefore also think of the number of inducing points as measuring the expressive capacity of the variational family.
This naturally raises the question: what happens to numerical stability when the number of data points grows, and expressive capacity grows with it?

To understand this, we decided to examine how stability and predictive performance change depending on data size and the number of inducing points simultaneously.
For this, we randomly sampled subsets from the East Africa dataset used in the preceding illustrative example, and computed both the respective kernel matrix condition number, and predictive root-mean-squared error.

Results can be seen in \Cref{fig:data-size}.
For a fixed data size (\Cref{fig:data-size}, top-to-bottom), we see that using more inducing points improves performance, but makes stability worse, with the tradeoff determined by the individual method: numerically stable methods favor stability, whereas optimizing the inducing points favors performance but quickly leads to unstable linear algebra.
On the other hand, for a fixed number of inducing points (\Cref{fig:data-size}, left-to-right), using more data always improves performance, and can actually slightly improve stability.
In light of \Cref{cor:condition-number-bound}, this makes sense, because using more data provides a larger region within which the algorithms can place a fixed number of inducing points, allowing them to be spaced further apart from one another.

If we grow both the number of data points and the number of inducing points at the same time while keeping their ratio approximately fixed (\Cref{fig:data-size}, diagonal, bottom-left to top-right), we see that numerical stability deteriorates as data size and the number of inducing points increase simultaneously.
The severity of the tradeoff depends on the algorithm used: stable methods provide a tradeoff which favors numerical stability, while optimized inducing points favor performance and become unstable faster. 
$K$-means achieves the worst of both worlds, having a level of performance comparable to stable methods, and a level of numerical stability comparable to optimizing the inducing points.
This highlights the value of methods designed to produce a favorable performance-stability tradeoff.

\section{Conclusion}
We study numerical stability of variational approximations used for Gaussian process models.
We survey and synthesize a number of results which, in total, enable one to conclude that the \emph{minimum separation distance} between inducing points controls stability properties of the numerical linear algebra needed to train variational approximations.
We introduce two techniques for ensuring stability: inducing points constructed by a \emph{cover tree}, and the \emph{clustered-data approximation}.
We provide examples illustrating the tradeoff between performance, stability, and dataset size.
Compared to standard inducing point placement methods, we find that approaches that promote minimum separation produce a favorable performance-stability tradeoff.

\section*{Acknowledgments}

DRB was supported by the Qualcomm Innovation Fellowship.
SF was supported by the EPSRC (EP/V002910/2).
We further acknowledge support from Huawei Research and Development, the Imperial College COVID-19 Research Fund, and UK Research and Innovation. 

\printbibliography
 
\newpage

\appendix

\section{Eigenvalues and Condition Numbers}
\label{apdx:eig-theory}

Here we analyze the behavior eigenvalues of kernel matrices, and how this affects numerical linear algebra.

\subsection{Covariance Operator Limits}

\PropCovOpLimit*

\begin{proof}
Define the (random) matrix $\m{\tl{K}}_{\v{x}\v{x}}$ to be zero on the diagonal and equal to $\m{\tl{K}}_{\v{x}\v{x}}$ on the off-diagonal. 
Let $\Pi$ denotes the set of all permutations of the natural numbers.
For an operator $\c{A}$ with discrete spectrum, we view the spectrum as an element of $\ell^p / \Pi$, denoted by $\lambda(\c{A})$.
By convention, we extend the spectrum of a symmetric, positive definite matrix $\m{M}$ to an element of $\ell^p / \Pi$ by appending zeros, denoted by $\lambda(\m{M})$.
For $p \in (1,\infty]$ define the the $\ell^p$ rearrangement distance by,
\[
d_p(x, y) = \inf_{\pi \in \Pi} \del{\sum_{i=1}^\infty |x_i - y_{\pi(i)}|^{p}}^{1/p},
\]
with the unusual convention that $p=\infty$ corresponds to the element-wise supremum. 
One can show that $d_p$ defines a pseudo-metric on $\ell^p$ and a metric on $\ell^p/\Pi$: the only property that is tricky to check is triangle inequality, which is shown in \textcite[p.~116]{koltchinskii2000}, in the case $p=2$, which immediately generalizes to other $p$.
Since $k$ is stationary, $\m{\tl{K}}_{\v{x}\v{x}} = \m{K}_{\v{x}\v{x}} - c \m{I}$ for some $c>0$. Therefore,
\[
d_\infty \del{\lambda(\tfrac{1}{N}\m{\tl{K}}_{\v{x}\v{x}}), \lambda(\tfrac{1}{N}\m{K}_{\v{x}\v{x}})} = \frac{c}{N} \-> 0.
\]
Since $k$ is bounded, the conditions of \textcite[Corollary 3.3]{koltchinskii2000} are satisfied, so
\[
d_2\del{\lambda(\tfrac{1}{N}\m{\tl{K}}_{\v{x}\v{x}}), \lambda(\c{K})} \-> 0. 
\]
and therefore by the standard $(\ell^2,\ell^\infty)$ inequality
\[
d_\infty\del{\lambda(\tfrac{1}{N}\m{\tl{K}}_{\v{x}\v{x}}), \lambda(\c{K})} \-> 0
\]
which implies that by the triangle inequality 
\[
d_\infty\del[2]{\lambda(\tfrac{1}{N}\m{K}_{\v{x}\v{x}}),\lambda(\c{K})} \-> 0
.
\]
We have
\[
\abs[2]{\lambda_1(\tfrac{1}{N}\m{K}_{\v{x}\v{x}}) - \lambda_1(\c{K})} \leq d_\infty\del[2]{\lambda(\tfrac{1}{N}\m{K}_{\v{x}\v{x}}),\lambda(\c{K})} \-> 0,
\]
so by the reverse triangle inequality
\[
\lambda_1(\m{K}_{\v{x}\v{x}}) = N \lambda_1(\c{K}) + o(N)
.
\]
As $\c{K}$ is not the zero operator and is positive definite, this is $\Omega(N)$. 
On the other hand, there exists an $N_0$ such that for all $n > N_0$, we have $\lambda_n(\c{K}) \leq \eps/2$.
Also, there exists an $N_0'$ such that for all $n\geq N_0'$
\[
d_\infty(\lambda(\tfrac{1}{N}\m{K}_{\v{x}\v{x}}),\lambda(\c{K})) \leq \eps/2.
\]
Choose $M_0 = \max(N_0, N_0')$.
Then for all $n \geq M_0$, 
\[
\frac{1}{N}\lambda_N(\m{K}_{\v{x}\v{x}}) \leq d_\infty(\lambda(\tfrac{1}{N}\m{K}_{\v{x}\v{x}}),\lambda(\c{K})) + \lambda_n(\c{K}) \leq \eps
\]
and so $\lambda_N(\m{K}_{\v{x}\v{x}})=o(N)$.
In fact, we can say much more: as the data is compactly supported, for any $\eps > 0$ and for all $N$ sufficient large there exists a $1 \leq i, j \leq N$ such that $k(x_i, x_j) \geq (1-\eps)k(x_i, x_i)$. 
By Cauchy's Interlacing Theorem,
\[
\lambda_N(\m{K}_{\v{x}\v{x}}) \leq \lambda_2
\begin{pmatrix}
k(x_i, x_i) & k(x_i, x_j) \\
k(x_j, x_i) & k(x_j, x_j)
\end{pmatrix}
\leq k(x_i,x_i)\eps.
\]
As $\eps$ is arbitrary, we conclude $\lambda_N(\m{K}_{\v{x}\v{x}}) \-> 0$.
\end{proof}

\subsection{Convergence of Conjugate Gradients}\label{apdx:convergence-of-cg}

\begin{result}
\label{res:cg-condition-number-bound}
Let $\v{v^*}$ satisfy $\m{A}\v{v^*} =\v{b}$. For any $t \in \mathbb{N}$ and $\v{v_0} \in \R^n$ and $\v{v_t}$ denote the solution found by running conjugate gradients on this system of equations with initial vector $\v{v_0}$, then

\[
 \|\v{v^*} - \v{v_t}\|_{\m{A}} \leq 2 \left(1-\frac{2}{\sqrt{\f{cond}(\m{A})}+1}\right)^{t}\|\v{v^*} - \v{v_0}\|_{\m{A}}.    
\]
\end{result}

\begin{proof}
\textcite[Theorem 10.2.6]{golub96}.
\end{proof}

\begin{corollary}\label{cor:cg-iterations}
A sufficient condition for  $\|\v{v^*} - \v{v_t}\|_{\m{A}}  \leq \eps$ is
\[
\label{eqn:iterations-upper-bound}
t \geq \del{\log\del[2]{1+\tfrac{2}{\sqrt{\f{cond}(\m{A})}+1}}}^{-1} \log\frac{2\|\v{v^*}- \v{v_0}\|_{\m{A}}}{\eps} = \left(\tfrac{\sqrt{\f{cond}(\m{A})}}{2}+\c{O}(1)\right)\log\frac{2\|\v{v^*} - \v{v_0}\|_{\m{A}}}{\eps}
\]
where the implicit constant is between $\frac{1}2$ and $\frac{3}{2}$.
\end{corollary}

\begin{proof}
The first inequality is immediate from inverting the bound in \Cref{res:cg-condition-number-bound}. The asymptotic formula can be derived directly from applying the inequalities $1-\frac{1}{x} \leq \log x \leq x-1$ to yield
\[
\frac{\sqrt{\f{cond}(\m{A})}+1}{2} \leq \del{\log\del[2]{1+\tfrac{2}{\sqrt{\f{cond}(\m{A})}+1}}}^{-1} \leq \frac{\sqrt{\f{cond}(\m{A})}+3}{2}.
\]
\end{proof}
\subsection{Upper Eigenvalue Estimates}

Here we prove the eigenvalue upper bound under separation and spatial decay via a packing argument.
We begin with a basic eigenvalue estimate.

\begin{result}[Gershgorin's Circle Theorem]
\label{res:gershgorin-circle-thm}
Let $\m{K}$ be a positive semi-definite symmetric matrix with eigenvalues $\lambda(\m{K}) = 
\{\lambda_i\}$.
Define the \emph{Gershgorin radius} $R_i = \sum_{j \neq i} |k_{ij}|$ to be the sum of the off-diagonal entries of $\m{K}$.
Then $\lambda(\m{K})\in \bigcup_{i}[k_{ii} - R_i, k_{ii} + R_i]$.
\end{result}
\begin{proof}
  \textcite[Theorem 8.1.3]{golub96}.
\end{proof}

We can use this result to get upper eigenvalue estimates as 
\[
\lambda_{\max} \leq \max_i k_{ii} + R_i    
\]
given a set of points that are separated and a kernel satisfying spatial decay. 
We now introduce these notions.

\begin{definition}
We say that a set of points $\v{x} \subset \R^{d}$ are $\delta$-\emph{separated} if, for every $x_i \neq x_j$ we have $\|x_i - x_j\| \geq \delta >0$.
\end{definition}

\begin{definition}
We define the ball of radius $\delta$-centered at $x \in \R^{d}$ to be the open set,
\[
B_{x, \delta} = \{x' \in X: \|x -x'\|< \delta\}.
\]
We write $B_{\delta}=B_{0, \delta}$.
\end{definition}
\AsmSpatialDecay*

The argument for the maximum eigenvalue bound is essentially that if we have spatial decay and uniformly bounded variance, then the kernel matrix has bounded diagonal and bounded Gershgorin radius, both independent of $N$.

\begin{definition}[Covering and packing numbers]
\label{defn:covering-packing-numbers}
Let $X \subset \R^{d}$ and let $\delta > 0$.
Define the following.
\1 A subset $C \subseteq C' \subseteq \R^{d}$ is an \emph{external cover} of $X$ in $C'$ if for every $x \in X$ there is a $c \in C$ such that $\norm{x - c} \leq \delta$.
\2 A subset $C \subseteq X$ is an \emph{internal cover} of $X$ if for every $x \in X$ there is a $c \in C$ such that $\norm{x - c} \leq \delta$.
\3 A subset $P \subseteq X$ is an \emph{external packing} of $X$ if any distinct $x, x' \in P$ satisfy $\norm{x - x'} > \delta$.
\4 A subset $P \subseteq X$ is an \emph{internal packing} of $X$ if it is an external packing, and moreover satisfies $B_{x,\delta} \subseteq X$ for all $x \in P$.
\0 Define the internal and external covering numbers $C^{\f{int},\delta}_X$ and $C^{\f{ext},\delta}_{X,C'}$ to be the smallest cardinality of such sets, and the internal and external packing numbers $P^{\f{int},\delta}_X$ and $P^{\f{ext},\delta}_X$ to be the largest cardinality of such sets.
\end{definition}

\begin{lemma}
Let $A_{m,\delta} = B_{0,(m+1)\delta} \takeaway B_{0,m\delta} = \{x \in \R^d : m\delta \leq \norm{x} < (m+1)\delta\}$ be an annulus with external radius $(m+1)\delta$ and internal radius $m\delta$.
Then we have
\[
\vol(B_{\delta}) &= \dfrac{\pi^{d/2}}{\Gamma(\frac{d}{2} + 1)} \delta^d
&
\vol(A_{m,\delta}) = \frac{\pi^{d/2}}{\Gamma(\frac{d}{2} + 1)} ((m+1)^d - m^d) \delta^d
.
\]
We also have for $\beta \geq 2$ as well as $\beta = \frac{1}{2}$ and $\beta = \frac{3}{2}$ that
\[
\label{eqn:binomial-diff}
(\beta+2)^d - \beta^d \leq 5^d \beta^{d-1}
. 
\]
\end{lemma}

\begin{proof}
The first part of the claim is standard.
For \eqref{eqn:binomial-diff}, suppose first that $\beta\geq 2$, and expand the first term on the left hand side with the Binomial Theorem. 
The lead term cancels with the second term on the left hand side, and we are left with a total of $2^d - 1$ non-negative terms, among which, using the assumption that $\beta\geq 2$, the largest is $2 \beta^{d-1}$.
Combining these estimates gives $(\beta+2)^d - \beta^d \leq (2^{d+1} - 2) \beta^{d-1}$, which using $2^{d+1} - 2 \leq 5^d$ gives the inequality.
We now handle the remaining cases: for $\beta = \frac{1}{2}$ we have $(\frac{5}{2})^d - (\frac{1}{2})^d \leq 5^d (\frac{1}{2})^{d-1}$, and for $\beta = \frac{3}{2}$, we have $(\frac{7}{2})^d - (\frac{3}{2})^d \leq 5^d (\frac{3}{2})^{d-1}$, both by elementary rearrangement of terms.
The claim follows.
\end{proof}

\begin{lemma}
For an integer $m\geq 1$, we have
\[
\del{\frac{1}{\delta}}^d \frac{\vol(A_{m,\delta})}{\vol(B_\delta)} \leq C^{\f{ext},\delta}_{A_{m,\delta}} \leq P^{\f{ext},\delta}_{A_{m,\delta}} \leq 5^d \del{\frac{2}{\delta}}^d \del{m - \frac{1}{2}}^{d-1}.
\]
and $P^{\f{ext},\delta}_{A_{0,\delta}} = 1$.
\end{lemma}

\begin{proof}
The first part of the claim follows immediately from \textcite[Lemma 5.13]{vanhandel14}.
To bound $P^{\f{ext},\delta}_{A_{m,\delta}}$ for $m \geq 1$, use \textcite[Proposition 4.2.12]{vershynin18}, along with the fact that the Minkowski sum of an annulus with a ball is another annulus, and the previous lemma, to bound
\[
P^{\f{ext},\delta}_{A_{m,\delta}} \leq \frac{\vol(B_{(m+\frac{3}{2})\delta} \takeaway B_{(m-\frac{1}{2})\delta})}{\vol(B_{\delta/2})} 
= \frac{\del{m+\frac{3}{2}}^d - \del{m-\frac{1}{2}}^d}{\del{\frac{\delta}{2}}^d} \leq 5^d \del{\frac{2}{\delta}}^{d} \del{m - \frac{1}{2}}^{d-1}
.
\]
For the case $m=0$, $A_{0,\delta}=B_{\delta}$ is non-empty, so $P^{\f{ext},\delta}_{A_{0,\delta}} \geq 1$. On the other hand, for any $x,x' \in B_{\delta}, \norm{x-x'} < \delta$, hence they cannot be part of the same $\delta$-packing.
\end{proof}

\begin{proposition}
\label{prop:eigenvalue-upper-bound-apdx}
Let $X = \R^d$, and let $k$ satisfy spatial decay.
Then there is a constant $C_{\max}^{k,\delta}$ such that for any $N$ and any $\delta$-separated $\v{x}$ of size $N$ we have
\[
\lambda_{\max}(\m{K}_{\v{x}\v{x}}) \leq C_{\max}^{k,\delta}
\]
\end{proposition}

\begin{proof}
Let $\psi$ be the upper-bounding function.
For an arbitrary $x_i$, we have 
\[
\lambda_{\max}(\m{K}_{\v{x}\v{x}}) \leq k_{ii} + R_i \leq \sum_{j=1}^N \psi(\norm{x_i  - x_j})
\]
so it suffices to bound this sum.
For that, we have 
\[
\sum_{j=1}^N \psi(\norm{x_i  - x_j}) = \sum_{m=0}^\infty \sum_{x_j \in A_{m,\delta}} \psi(\norm{x_i  - x_j}) \leq \sum_{m=0}^\infty P^{\f{ext},\delta}_{A_{m,\delta}} \psi(m\delta). 
\]
The first equality uses that $\{A_{m,\delta}\}_{m=0}^\infty$ forms a partition of $\R^{d}$. The second uses that for $x_j \in A_{m,\delta}$, we have $\psi(\norm{x_i  - x_j}) \leq \psi(m\delta)$, and the maximum number of $x_j$ in each $A_{m,\delta}$ is bounded by the external packing number $P^{\f{ext},\delta}_{A_{m,\delta}}$ by the definition of a $\delta$-separated set. 
By the preceding lemma, spatial decay, and the integral test for infinite series, we conclude 
\[
\sum_{m=0}^\infty P^{\f{ext},\delta}_{A_{m,\delta}} \psi(m\delta) \leq \psi(0)+ 5^d \del{\frac{2}{\delta}}^{d}  \sum_{m=1}^\infty \del{m - \frac{1}{2}}^{d-1} \psi(m\delta) = C_{\max}^{k,\delta} < \infty
.
\]
\end{proof}

\subsection{Lower Eigenvalue Estimates}

For lower bounds, we apply \textcite[Theorem 12.3]{wendland04}, which in the stationary case follows from our assumptions.
This argument, which builds on the work of \textcite{narcowich92,schaback95}, is somewhat involved. 
To aid understanding, rather than simply quote the result, below we present a sketch of its key steps.

The idea is to invoke the lower estimate given by Gershgorin's Circle Theorem---\Cref{res:gershgorin-circle-thm}---which was previously used for the upper bounds. 
However, applying this estimate directly leads in general to the inequality $\lambda_{\min}(\m{K}_{\v{x}\v{x}}) \geq \tl{C}_{\min}^{k,\delta}$ where potentially $\tl{C}_{\min}^{k,\delta} \leq 0$, which is vacuous.
To overcome this, the idea is to carefully construct a different covariance kernel $\phi$ whose eigenvalues lower-bound those of $k$, and which admit better Gershgorin estimates.
This is done via the following lemma.

\begin{lemma}
\label{lem:lower-bound-quadratic-form}
Let $k$ and $\phi$ be stationary kernels with respective spectral densities $\rho$ and $\varsigma$.
Suppose that $\varsigma \leq \rho$. 
Then for all $\v{x}$ we have
\[
\lambda_{\min}(\m\Phi_{\v{x}\v{x}}) \leq \lambda_{\min}(\m{K}_{\v{x}\v{x}})
.
\]
\end{lemma}

\begin{proof}
For any $\v{v} \in \R^{N}$ and by Bochner's Theorem---see \textcite[Theorem 6.6]{wendland04}---we have
\[
\v{v}^T \m{K}_{\v{x}\v{x}} \v{v} &= \sum_{i=1}^N \sum_{j=1}^N v_i k(0, x_i - x_j) v_j
\\
&= \sum_{i=1}^N \sum_{j=1}^N v_i v_j \int_{\R^d} e^{2\pi i \langle\omega,x_i - x_j\rangle} \d\rho(\omega)
\\
&= \int_{\R^d} \sum_{i=1}^N v_i e^{2\pi i \langle\omega, x_i\rangle} \sum_{j=1}^N v_j e^{2\pi i \langle\omega, x_j\rangle} \d\rho(\omega)
\\
&= \int_{\R^d} \abs[4]{\sum_{j=1}^N v_j e^{2\pi i\omega x_j}}^2 \d\rho(\omega) 
\\
&\geq \int_{\R^d} \abs[4]{\sum_{j=1}^N v_j e^{2\pi i\omega x_j}}^2 \d\varsigma(\omega)
\\
&= \v{v}^T\m\Phi_{\v{x}\v{x}}\v{v}
\]
where, in the third inequality, we use that since $v_j$ is real, it is equal to its complex conjugate. 
Taking an infimum over $\v{v}$ on the unit sphere on the left hand side gives the result.
\end{proof}

With this established, we can obtain lower eigenvalue estimates for any kernel $\phi$ whose Gershgorin radius is small enough, by invoking separation and using a packing argument similar to the one used to prove \Cref{prop:eigenvalue-upper-bound-apdx}.
This is done as follows.

\begin{lemma}
Let $X = \R^d$, and let $\phi$ decay sufficiently fast.
Then there is a constant $C^{\phi,\delta}_{\min}>0$ such that for any $N$ any any $\delta$-separated $\v{x}$ of size $N$ we have 
\[
\lambda_{\min}(\m\Phi_{\v{x}\v{x}}) \geq C^{\phi,\delta}_{\min}
.  
\]
\end{lemma}

\begin{proof}
For arbitrary $x_i$, letting $R_i$ be the Gershgorin radius, we have 
\[
\lambda_{\min}(\m\Phi_{\v{x}\v{x}}) \geq \Phi_{ii} - R_i
\]
so it suffices to upper-bound $R_i$ sufficiently tightly.
Applying the packing argument of \Cref{prop:eigenvalue-upper-bound}, we get 
\[
\sum_{j\neq i} \phi(x_i - x_j) \leq d \del{\frac{3}{\delta}}^d \sum_{m=1}^\infty  m^{d-1} \phi(m\delta)
\]
where we have substituted the lower-bounding kernel $\phi$ in to the argument, modified slightly to exclude the diagonal part, which corresponds to $m=0$.
The claim follows, provided
\[
d \del{\frac{3}{\delta}}^d \sum_{m=1}^\infty  m^{d-1} \phi(m\delta) < \phi(0)
\]
which holds provided that $\phi$ decays sufficiently fast.
\end{proof}

From here, all that remains is to show that for any stationary kernel $k$, it is always possible to choose a $\phi$ that decays sufficiently fast.
\textcite{wendland04} shows that for every kernel $k$ there is a constant $0 < M_k < \infty$ for which taking the spectral density $\varsigma$ of $\phi$ to be the convolution of an indicator function of the ball $B_{M_k}$ with itself, namely
\[
\varsigma \propto \1_{B_{M_k}} * \1_{B_{M_k}}
\]
with proportionality scaling chosen to ensure the spectral lower-bound needed for \Cref{lem:lower-bound-quadratic-form} holds, suffices to yield a non-vacuous bound.
The idea behind this choice, which dates back at least to ideas of \textcite{narcowich92,schaback95} in the polynomial interpolation literature, is to leverage the uncertainty principle property of the Fourier transform.
Specifically, any kernel concentrated enough to decay sufficiently fast in the spatial domain must possess a sufficiently diffuse spectral measure.
The simplest choice is to take $\varsigma \propto \1_{B_{M_k}}$, but this fails due to issues with smoothness, which are fixed by applying a convolution and rescaling appropriately, yielding the presented choice.

At this point, the main technical work that remains is to calculate the form of the resulting kernel $\phi$, and deduce how to choose the constant $M_k$ and proportionality scaling to ensure it both lower-bounds $k$ and gives a non-vacuous estimate.
For that, we refer to \textcite[Ch. 12.2]{wendland04}.

\begin{result}
\label{prop:eigenvalue-lower-bound}
Let $X = \R^d$, and let $k$ be stationary.
Then there is a constant $C_{\min}^{k,\delta}$ such that for any $N$ and any $\delta$-separated $\v{x}$ of size $N$ we have
\[
\lambda_{\min}(\m{K}_{\v{x}\v{x}}) \geq C_{\min}^{k,\delta}.
\]
\end{result}

\begin{proof}
\textcite[Theorem 12.3]{wendland04}.
\end{proof}
\subsection{Settings where Separation is a Necessary Condition}
\label{prop:apdx-separation-necessary}

Let $k$ be a kernel.
We say that $k$ is a \emph{Lipschitz continuous kernel} if $k$ is Lipschitz in both arguments separately.
We extend the notion of minimum separation from finite vectors to infinite sequences in the obvious pairwise manner.

\PropSeparationNecessary*

\begin{proof}
We argue by contradiction: suppose minimum separation does not hold.
We will show that there does not exist a $C$ such that, for all $M$, $\f{cond}(\m{K}_{\v{z}\v{z}}^{(M)}) \leq C$. 
By Cauchy's Interlacing Theorem, any submatrix of $\m{K}_{\v{z}\v{z}}^{(M)}$ has condition number no larger than $\f{cond}(\m{K}_{\v{z}\v{z}}^{(M)})$. 
Hence, it suffices to find a sequence of submatrices $\m{K}_m$, for $m=1,..,\infty$, such that the rows and columns of $\m{K}_m$ are contained in the rows and columns of $\m{K}_{\v{z}\v{z}}^{(m)}$, and show that this sequence has unbounded condition number. 
We therefore consider such sequences, consisting of $2 \x 2$ matrices, in the remainder of the proof.
  
We now split into two cases. 
The first case handles when the variance of the kernel can differ by an arbitrary factor at different points in input space. This is handled as a separate case, as the technique in the second, arguably more intuitive case, will break down if the variance can be arbitrarily close to zero. 
  
\emph{Case I: for any $C>0$, there exists an $i,j$ with $\frac{k(z_i,z_i)}{k(z_j,z_j)} > C$}. Let $\m{K}_m$ be the matrix formed by 
\[
z^{(m)}_i &= \argmax_{i' \leq m} k(z_{i'}, z_{i'})
&
z^{(m)}_j &= \argmin_{j' \leq m} k(z_{j'}, z_{j'})
\] 
where we take take the smallest such $i,j$ if the maximum and minimum are not unique.
Then
\[
\f{cond}(\m{K}_m) \geq \f{cond}\begin{pmatrix}
k(z^{(m)}_i, z^{(m)}_i) & 0 \\
0 & k(z^{(m)}_j, z^{(m)}_j)
\end{pmatrix} = \frac{k(z^{(m)}_i, z^{(m)}_i)}{k(z^{(m)}_j, z^{(m)}_j)}
.
\]
The inequality can be verified by inspection using the formula for eigenvalues of a $2 \x 2$ matrix. 
By the assumption in this case, this ratio tends to infinity with $m$, contradicting that the condition number is bounded.

\emph{Case II: there exists a $C>0$ such that for all pairs $i,j$ we have $\frac{k(z_i,z_i)}{k(z_j, z_j)} \leq C$.} 
By the assumption in this case, for all $i$, we have 
\[ 
0 < \frac{1}{C}k(z_1, z_1) \leq k(z_i,z_i) \leq Ck(z_1, z_1) 
.
\] 
Define $a= \frac{1}{C}k(z_1, z_1)$. Take $z^{(m)}_i, z^{(m)}_j$ to be any pair such that $\norm[0]{z^{(m)}_i-z^{(m)}_j} \leq \norm{z_{i'}- z_{j'}}$ for any pair $i', j' \leq m$, and $\m{K}_m$ to be the submatrix associated to this pair of points. 
By Lipschitzness of the kernel, there exists an $L$ such that  
\[
|k(z^{(m)}_i, z^{(m)}_i) - k(z^{(m)}_i, z^{(m)}_j)| &\leq L \f{sep}_m
&
|k(z^{(m)}_i, z^{(m)}_i) - k(z^{(m)}_k, z^{(m)}_j)| &\leq 2L \f{sep}_m
.
\]
where $\f{sep}_m$ denotes the separation distance of $(z_1,..,z_m)$.
Combining these estimates gives the Frobenius norm bound
\[
\norm{\frac{1}{k(z_i, z_i)}\m{K}_m - \m{1}}_{F} \leq \frac{L\f{sep}_m\sqrt{6}}{k(z_i, z_i)} \leq \frac{L\f{sep}_m\sqrt{6}}{a} ,
\]
where $\m{1}$ is the $2 \x 2$ matrix of all ones. 
As the condition number of $\frac{1}{k(z_i, z_i)}\m{K}_m$ and $\m{K}_m$ are the same, we study the the condition number of $\frac{1}{k(z_i, z_i)}\m{K}_m$. 
The operator norm is less than the Frobenius norm, so by reverse triangle inequality applied to the operator norm
\[ 
\lambda_1\del{\frac{1}{k(z_i, z_i)}\m{K}_m} \geq 2 - \norm{\frac{1}{k(z_i, z_i)}\m{K}_m - \m{1}}_{2} \geq 2- \frac{L\f{sep}_m\sqrt{6}}{a}
\]    
and 
\[
\lambda_2\del{\frac{1}{k(z_i, z_i)}\m{K}_m} \leq \frac{L\f{sep}_m\sqrt{6}}{a}
.
\]
By our assumption that $\f{sep}_m$ is not bounded below, this ratio of these tends to infinity with $m$.
\end{proof}

\subsection{Effect of Dimension on Condition Number}\label{apdx:condition-number-dimension}

We now consider the simple case of independent inputs in $\R^{d}$ with independent and identically distributed dimensions. 
The key observation is that for large $d$, by concentration of measure, we should most pairwise distances to concentrate around the average distance.
This means that there will not be many tightly clustered sets of data points. 
Instead, the data will be automatically well-separated with high probability in this specific setting.

Let $\rho$ be the distribution of each input. 
We work with uniformly generated inputs for \Cref{fig:spatial-res-results}, specifically $\rho = \f{U}(-2\pi, 2\pi)$. 
For any $i \neq j$ we have
\[
\norm{x_i - x_j}^2_2 = \sum_{\ell=1}^d \Delta_\ell^2,  
\]
with $\Delta_\ell = (\xi_{i,\ell} - \xi_{j,\ell})^2$ where $\xi_{i,\ell}\~\rho$ are all independent, hence $\Delta_\ell$ are also all independent.

If we consider the isotropic squared exponential kernel with length scale $\sqrt{d}$, which is used for generating \Cref{fig:spatial-res-results}, the off-diagonal entries of the kernel matrix are independently and identically distributed random variables of the form
\[
k_{ij} = \exp\del{-\frac{1}{2d}\sum_{\ell=1}^d \Delta_\ell^2}
.
\]
By the Law of Large Numbers, as $d \-> \infty$, we have
\[
\frac{1}{2d}\sum_{\ell=1}^d \Delta_\ell^2 \-> \frac{1}{2}\E(\Delta_\ell^2),
\]
and so by the Continuous Mapping Theorem we get
\[
k_{ij} \-> \exp\del{-\frac{1}{2}\E(\Delta_\ell^2)}
\]
for large $d$. 
On the other hand, the diagonal entries of the kernel matrix are all surely $1$.

In the case of \Cref{fig:spatial-res-results}, we have $(x_i)_\ell \~[U](-2\pi, 2\pi)$. Hence,
\[
\E(\Delta_\ell^2) = \frac{4}{3}\pi^2,
\]
so for $d$ large, we have $k_{ij} \-> \exp\del{-\frac{2\pi^2}{3}}$.
For fixed $N$ and large $d$, we have the approximation
\[
\m{K}_{\v{x}\v{x}} \-> \del{1- \exp\del{-\frac{2\pi^2}{3}}}\m{I} + \exp\del{-\frac{2\pi^2}{3}}\v{1}\v{1}^T.
\]
The spectrum of the matrix on the right hand side is 
\[
\lambda(\m{K}_{\v{x}\v{x}}) = \del{1 +(N-1)\exp(-\tfrac{2\pi^2}{3}), 1- \exp(-\tfrac{2\pi^2}{3}),.., 1- \exp(-\tfrac{2\pi^2}{3})}
.
\]
Since, in our experiments, $\exp(-\tfrac{2\pi^2}{3}) \approx \frac{1}{N}$, the condition number of $\m{K}_{\v{x}\v{x}}$ will be small.
Since the entries of $\m\Lambda$ are small, this implies the condition number of $\m{K}_{\v{x}\v{x}} + \m{\Lambda}$ will be small.
While this argument only holds asymptotically, one might expect the Central Limit Theorem to apply relatively quickly since the inputs are light-tailed, thus explaining observed behavior for, say, $d=8$ in \Cref{fig:spatial-res-results}.

\section{Inducing Points}
\label{apdx:inducing-points}

Here we develop the inducing-point-specific theory used in the manuscript.

\subsection{Clustering-based Inducing Points}

\PropInducingPointRepr*

\begin{proof}
The likelihood can be written 
\[
p(y_i \given f(x_i)) \propto \exp\del{- \sum_{i=1}^N \frac{(y_i - f(\f{cl}(x_i)))^2}{2\sigma^2}} = \exp\del{- \sum_{j=1}^M \sum_{\f{cl}(x_i) = z_j} \frac{(y_i - f(z_j))^2}{2\sigma^2}}
\]
and the claim follows by noting that by Gaussianity, each $\del{u_j, \frac{1}{N_{\f{cl}}(z_j)}}$ pair is a sufficient statistic for all $y_i$ whose indices satisfy $\f{cl}(x_i) = z_j$.
\end{proof}

\subsection{Stochastic Maximum Marginal Likelihood}
\label{apdx:training}

We train hyperparameters using maximum marginal likelihood via doubly stochastic gradient descent.
The minimization problem resembles the usual variational inference problem in most inducing point approaches, and is
\[
\argmin_{q_f\in\bb{Q}} D_{\f{KL}}(q_f \from \pi_f) - \frac{N}{2}\log(2\pi\sigma^2) -  \frac{1}{2\sigma^2} \E*_{f\~q_f} (\v{y} - f(\v{x}))^T (\v{y} - f(\v{x}))
\]
where $\bb{Q}$ is the space of all Gaussian processes of the form 
\[
(f\given\v{u})(\.) = f(\.) + \m{K}_{(\.)\v{z}}(\m{K}_{\v{z}\v{z}} + \m\Lambda)^{-1}(\v{u} - f(\v{z}) - \v\eps)
\]
and $\v{z} = \f{cl}(\v{x})$, $u_j = \frac{1}{N_{\f{cl}}(z_j)}\sum_{\f{cl}(x_i) = z_j} y_i$, and $\Lambda_{ii} = \frac{\sigma^2}{N_{\f{cl}}(z_i)}$.
The only parameters which are minimized numerically are the kernel hyperparameters.
We simplify the variational expectation term using the identity
\[
\E (\v{y} - f(\v{x}))^T (\v{y} - f(\v{x})) = \sum_{i=1}^N \E (y_i - f(x_i))^2 = \sum_{i=1}^N (y_i - \E(f(x_i)))^2 + \Var(f(x_i))
.
\]
To train the process, we need to do two things: compute the prior Kullback--Leibler divergence, and compute the likelihood term.
For the latter, we sample a mini-batch of data, and use it compute an unbiased estimator of the sum by evaluating the inner terms on the mini-batch.
For a mini-batch $\v{x}'$, the mean and variance terms are
\[
\E(f(\v{x}')) &= \m{K}_{\v{x}'\v{z}} (\m{K}_{\v{z}\v{z}} + \m\Lambda)^{-1} \v{u}
&
\Var(f(\v{x}')) &= \m{K}_{\v{x}' \v{z}} (\m{K}_{\v{z}\v{z}} + \m\Lambda)^{-1} \m{K}_{\v{z}\v{x}'}  
\]
which can be computed using matrix-vector products whose cost, in the case of the variance, is quadratic for each data point in the mini-batch.
For the prior term, we have the identity 
\[
2 D_{\f{KL}}(q_f \from \pi_f) &= \ln\frac{|\m{K}_{\v{z}\v{z}}|}{|\m{K}_{\v{z}\v{z}} - \m{K}_{\v{z}\v{z}}(\m{K}_{\v{z}\v{z}} + \m\Lambda)^{-1}\m{K}_{\v{z}\v{z}}|} - d 
\\
&\quad+ \tr(\m{K}_{\v{z}\v{z}}^{-1}(\m{K}_{\v{z}\v{z}} - \m{K}_{\v{z}\v{z}}(\m{K}_{\v{z}\v{z}} + \m\Lambda)^{-1}\m{K}_{\v{z}\v{z}}))
\\
&\quad + \v{u}^T(\m{K}_{\v{z}\v{z}} + \m\Lambda)^{-1}\m{K}_{\v{z}\v{z}}\m{K}_{\v{z}\v{z}}^{-1}\m{K}_{\v{z}\v{z}}(\m{K}_{\v{z}\v{z}} + \m\Lambda)^{-1}\v{u}
\\
&= \ln\frac{|\m{K}_{\v{z}\v{z}} + \m\Lambda|}{|\m\Lambda|} - \tr((\m{K}_{\v{z}\v{z}} + \m\Lambda)^{-1}\m{K}_{\v{z}\v{z}}) \label{eq:prior-kl-trace-estimator}
\\
&\quad + \v{u}^T(\m{K}_{\v{z}\v{z}} + \m\Lambda)^{-1}\m{K}_{\v{z}\v{z}}(\m{K}_{\v{z}\v{z}} + \m\Lambda)^{-1}\v{u}
.
\]
This means that all linear systems that need to be solved involve the well-conditioned matrix $(\m{K}_{\v{z}\v{z}} + \m\Lambda)^{-1}$.
Unfortunately, $(\m{K}_{\v{z}\v{z}} + \m\Lambda)^{-1}\m{K}_{\v{z}\v{z}}$ has a matrix-valued rather than vector-valued right-hand-side, and thus requires $M$ total matrix-vector products, leading to cubic costs, even with conjugate gradients.
To avoid this, we apply Hutchinson's trace estimator, which we now describe.

To illustrate Hutchinson's trace estimator, let $\m{A}$ an arbitrary matrix, for instance $\m{A} = (\m{K}_{\v{z}\v{z}} + \m\Lambda)^{-1}\m{K}_{\v{z}\v{z}}$.
We are interested in evaluating $\tr(\m{A})$ in cases where we do not have direct access to $\m{A}$, but can compute matrix-vector products with $\m{A}$. \textcite{hutchinson1989stochastic} observed that for any vector $\v{v}$ such that $\E(\v{v}\v{v}^T) = \m{I}$ we have
\[
\tr(\m{A}) = \tr(\E(\v{v}\v{v}^T)\m{A}) = \E(\v{v}^T \m{A}\v{v}).
\]
This motivates one to employ the unbiased Monte Carlo estimator
\[
\frac{1}{L} \sum_{\ell=1}^L \v{v}_{\ell}^T \m{A}\v{v}_{\ell}
\]
where the $\v{v}_{\ell}$ are independent and identically distributed. Typically, $\v{v}_{\ell}$ is chosen to be Rademacher distributed, where each entry of $\v{v}$ is independent and takes the values $\{-1,1\}$ with equal probability. 
This choice of $\v{v}$ is leads to the minimum variance---see \textcite[Proposition 1]{hutchinson1989stochastic}.

In the case $\m{A} = (\m{K}_{\v{z}\v{z}} + \m\Lambda)^{-1}\m{K}_{\v{z}\v{z}}$, we sample the Rademacher random variables, solve the linear system $(\m{K}_{\v{z}\v{z}} + \m\Lambda)^{-1}\v{v}$, then multiply the result by $\m{K}_{\v{z}\v{z}}\v{v}$. 
The complexity of the matrix-vector products required for both operations is quadratic. 

The log-determinant term is handled similarly.
Rather than evaluate it directly, we follow \textcite{gibbs1997efficient} and apply stochastic estimation on its gradient.
Using that
\[
\pd{ |\m{K}_{\v{z}\v{z}} + \m\Lambda|}{ \theta} = \tr\del{\pd{ (\m{K}_{\v{z}\v{z}} + \m\Lambda)}{ \theta}(\m{K}_{\v{z}\v{z}} + \m\Lambda)^{-1}}
\]
we apply Hutchinson's trace estimator as in the case of the trace discussed above.

\section{Cover Trees}
\label{apdx:cover-tree}

To prove the key claim, we formalize the cover tree as a labeled rooted tree in the sense of a directed graph, defined inductively by the steps of the algorithm.
We distinguish individual nodes from their locations in space, which we view as labels.

\begin{definition}
Let $T = (V,E,v_0)$ be a rooted tree, with root node $v_0 \in V$.
Define the \emph{tree order} $\prec$ by $v \prec v'$ if the unique path from $v_0$ to $v'$ passes through $v$.
Define the \emph{parent} of a node to be 
\[
\f{pa} : V\takeaway\{v_0\} \-> V
\]
which maps each $v$ to the unique vertex $\f{pa}(v) \prec v$ for which $(v,\f{pa}(v)) \in E$, which is well-defined because $T$ is a rooted tree.
Define the \emph{children} of a node to be 
\[
\f{ch} &: V \-> 2^V
&
\f{ch}(v) &= \f{pa}^{-1}(\{v\})
\]
which is the set-theoretic preimage of the parent function.
Define the \emph{level}
\[
\f{lv} : V \-> \N_0
\]
of a node to be function which maps each $v$ to the length $\f{lv}(v)$ of the path from $v_0$ to $v$, which is also well-defined because $T$ is a rooted tree.
\end{definition}

With these notions, we can define an intermediate state of the cover tree algorithm.

\begin{definition}
We say that a \emph{cover tree} is a tuple $(T,z,\c{A})$, where $T$ is a rooted tree and
\[
z &: V \-> \v{x}
&
\c{A} &: V \-> 2^{\v{x}}
\]
are, respectively, the \emph{spatial location} and \emph{assigned data} of each node.
\end{definition}

The cover tree algorithm, then, is a function which maps a cover tree into another cover tree, by constructing the next level.
Our goal is to prove that this algorithm terminates in finite time, producing a rooted tree whose nodes satisfy the desired properties.
We analyze a more general form of the algorithm parameterized by a constant $a > 1$ and sequence $b_\ell > 0$, where $a$ determines the rate at which the radius of each ball decays, and $b_\ell$ determines the size of an $R$-neighbor ball at level $\ell$, with $\Cref{alg:covertree}$ corresponding to $a = 2$ and $b_\ell = 4(1 - 1/2^\ell)R_0$, which will be derived as the optimal sequence for the given choice of $a$.

\begin{definition}
Let $V_0 = \{v_0\}$, $E_0 = \emptyset$, $z_0(v_0) = \frac{1}{N} \sum_{i=1}^n x_i$, and $\c{A}_0 = \v{x}$.
Let $(T_\ell,z_\ell,\c{A}_\ell)$ be the output of one step of \Cref{alg:covertree}, defined inductively starting from $\ell=0$.
\end{definition}

From this, we can define the inducing points produced by the algorithm.

\begin{definition}
For a given tree $T$, define the \emph{level-$\ell$ nodes} and \emph{level-$\ell$ inducing points}
\[
V_\ell(T) &= \{v\in V : \f{lv}(v) = \ell\}    
&
\v{z}_\ell(T) &= \{z(v) : v \in V : \f{lv}(v) = \ell\}
.    
\]
\end{definition}

We first verify that the nodes and hence inducing points at a fixed level remain fixed as the algorithm proceeds.

\begin{lemma}\label{lem:tree-levels-preserved}
For $\ell' \geq \ell$, we have 
\[
V_\ell(T_{\ell'}) &= V_\ell(T_\ell)
&
\f{sep}(\v{z}_\ell(T_{\ell'})) &= \f{sep}(\v{z}_\ell(T_\ell))
&
\f{res}_{\v{x}}(\v{z}_\ell(T_{\ell'})) &= \f{res}_{\v{x}}(\v{z}_\ell(T_\ell))
.
\]
\end{lemma}

\begin{proof}
The second and third claims follow from the first one, which is immediate by the construction of $\Cref{alg:covertree}$.
\end{proof}

Next, we check that the separation and spatial resolution properties hold at the root node.

\begin{lemma}
We have
\[
\f{sep}(\v{z}_0(T_0)) &= \infty
&
\f{res}_{\v{x}}(\v{z}_0(T_0)) \leq \max_{x\in \v{x}} \norm{x - z(v_0)}
.
\]
\end{lemma}

\begin{proof}
Because $V_0 = \{v_0\}$, the supremum in the definition of $\f{sep}$ is taken over an empty set, and the first claim follows. 
For the second claim, recall that $z(v_0) = \frac{1}{\v{x}} \sum_{x' \in \v{x}} x'$. 
    For any $x \in \v{x}$, 
    \[
        \norm[1]{x - z(v_0)}
         = \frac{1}{|\v{x}|}\norm[2]{\sum_{x'\in \v{x}} x - x'} 
         \leq  \frac{1}{|\v{x}|}\sum_{x'\in \v{x}}\norm{ x - x'}
        \leq  \max_{x' \in \v{x}}\norm{x - x'}.
    \]
Taking a maximum over $x \in \v{x}$ completes the proof.
\end{proof}

We now define the notion of $R$-neighbors.

\begin{definition}
Define $R_\ell = \frac{R_{\ell-1}}{a}$ inductively, where $R_0 = \max_{x\in \v{x}} \norm{x - z(v_0)}$.
For $b > 0$, define the \emph{$R_\ell$-neighbors} of a node inductively by 
\[
\c{R}(v_0) &= \{v_0\}
&
\c{R}(v) &= \cbr{v' \in \U_{s \in \c{R}(\f{pa}(v))} \f{ch}(s) : \norm{z(v) - z(v')} \leq b_{\ell} R_{\ell}}
\]
\end{definition}

The next step is to deduce what sequences $b_\ell$ are valid for a given value of $a$.

\begin{lemma}\label{lem:neighbor-lemma}
Let 
\[
b_{\ell} = \frac{2(1+c(\frac{1}{a})^{\ell})}{1-\frac{1}{a}}
.  
\] 
for $c \in \{-1, -a^{L}\}$.
Then, letting
\[
\c{R}'(v) = \cbr{v' \in V_{\ell}(T_{\ell}): \norm{z(v) - z(v')} \leq b_{\ell} R_{\ell}}
\]
be the set of $R_\ell$-nearby nodes, we have 
\[
\c{R}'(v) = \c{R}(v)
.
\]
\end{lemma}
\begin{proof}
First, note for all $v \in V_{\ell}(T_{\ell})$ that $\c{R}(v) \subseteq \c{R}'(v)$, since 
\[
\U_{s \in \c{R}(\f{pa}(v))} \f{ch}(s) \subseteq V_\ell(T_{\ell})    
\]
since the set of children of the parents of a node are contained in the same level as the original node.
It therefore suffices to show $\c{R}'(v) \subseteq \c{R}(v)$.
The proof of this proceeds by induction on depth. In the base case $\ell=0$ there is a single node $v_0$, and $v_0 \in \c{R}(v_0) = \{v_0\}$, so the proposition is true.
We take the inductive hypothesis that
\[
\c{R}'(v'') \subseteq \c{R}(v'')
\]
for all $v'' \in V_{\ell-1}(T_{\ell-1})=V_{\ell-1}(T_{\ell})$, where the equality follows from \Cref{lem:tree-levels-preserved}.
We now show the claim holds for  $v \in V_{\ell}(T_{\ell})$. 
Let 
\[
v' \in \c{R}'(v)    
\]
for which we would like to show that $v' \in \c{R}(v)$.
The strategy will be to consider these nodes' parents, and show that $\f{pa}(v') \in \c{R}(\f{pa}(v))$, in which case the result follows from the definition of $\c{R}$.
By the combination of the inductive hypothesis and first part of the proof, we have $\c{R}'(v'') = \c{R}(v'')$ for all $v'' \in V_{\ell-1}$, so it suffices to show
\[\label{eqn:ih-reduction}
\norm{z(\f{pa}(v')) - z(\f{pa}(v))} \leq b_{\ell-1}R_{\ell-1}.
\]
We therefore upper bound the distance between the locations of the parents of the nodes in question. We have
\[
\norm{z(\f{pa}(v')) - z(\f{pa}(v))} &\leq \norm{z(\f{pa}(v'))-z(v')} + \norm{z(\f{pa}(v)) - z(v)} + \norm{z(v') - z(v)} 
\\
&\leq \|z(\f{pa}(v')) - z(v')\| + \|z(\f{pa}(v)) - z(v)\| + b_{\ell} R_{\ell} 
\\
&\leq 2 R_{\ell-1} + b_{\ell} R_{\ell}.
\\
&= 2 R_{\ell-1} + \frac{b_{\ell}}{a} R_{\ell-1} \label{eqn:parent-distance-bdd}
.
\]
The first inequality by the triangle inequality.
The second inequality follows by definition of $v'$.
The third inequality uses that the locations of the children of a node are contained in the assigned data of the node, and that all points contained in the assigned data of a node in level $\ell - 1$ must be within a distance of $R_{\ell-1}$ of the label, or in equations that for all $v'' \in V_{\ell-1}(T_{\ell})$ and $x \in \c{A}_{\ell-1}(v)$ we have
\[
\norm{z(v'') - x} \leq R_{\ell-1}
.
\]
Combining \eqref{eqn:parent-distance-bdd} and \eqref{eqn:ih-reduction} gives non-homogenous linear recurrence relation
\[
2 + \frac{1}{a}b_{\ell} \leq b_{\ell-1}
.
\]
The solution can be derived using standard results on recurrence relations \cite[Chapter 2]{greene1990mathematics}, which gives
\[
b_{\ell} = \frac{2(1+c(\frac{1}{a})^{\ell})}{1-\frac{1}{a}}
\]
for some constant $c$ to be determined by the initial conditions---one can easily check by plugging in the formula that this satisfies the recurrence.
We require that $b_{\ell}\geq 0$ for $1 \leq \ell \leq L$, where $L$ is the maximum tree depth. 
The optimal choice of initial conditions given this constraint, in the sense of minimizing the size of neighborhoods while retaining guarantees, come from selecting $b_{0} = 0$ or $b_{L}=0$. 
The former gives $c=-1$, while the latter gives $c=-a^{L}$.
This completes the proof.
\end{proof}

We now argue that the spatial resolution of the algorithm satisfies the claimed inequality.

\begin{lemma}
We have
\[
\f{res}_{\v{x}}(\v{z}_\ell(T_\ell)) \leq R_\ell
.
\]
\end{lemma}

\begin{proof}
By induction on depth, one can show that for every level, $\{\c{A}_\ell(v)\}_{v \in V_{\ell}}$ is a partition of $\U_{i=1}^N \{x_i\}$, where $x_i$ are entries of $\v{x}$.
For $x \in X$, let $v_{x, \ell}$ denote the unique $v \in V_{\ell}$ such that $x \in \c{A}_\ell(v_{x, \ell})$. Then
\[
    \min_{v\in V_{\ell}}\| x - z(v)\| \leq \| x - z(v_{x,\ell})\| \leq R_{\ell},
\]
where the second inequality uses that the active set is contained in the ball of radius $R_{\ell}$ centered at $z(v_{x,\ell})$, which holds by definition of \Cref{alg:covertree}, specifically \cref{line:update-data}. 
If optional Voronoi repartitioning is performed, we further upper bound $\norm{x-z(v_{x,\ell})}$ by $\norm[0]{x-z(v'_{x,\ell})}$, where $v'_{x,\ell}$ denotes the vertex to which $x$ was assigned prior to the Voronoi step.
The results follows from taking a maximum over $x \in \v{x}$.
\end{proof}

Next, we argue the separation distance inequality holds.

\begin{lemma}
We have
\[
\f{sep}(\v{z}_\ell(T_\ell)) \geq R_\ell
.
\]
\end{lemma}

\begin{proof}
We prove the contrapositive: suppose that there exist a $v,v' \in V_{\ell}$ such that $\norm{z(v) - z(v')} \leq R_{\ell}$. 
By the reverse triangle inequality, this yields
\[
R_{\ell} &\geq \norm{z(v) - z(v')} 
\\
&\geq \norm{z(\f{pa}(v)) - z(\f{pa}(v'))} - \norm{z(\f{pa}(v)) -z(v)} - \norm{z(\f{pa}(v')) - z(v')}
.
\]
On the other hand, we have,
\[
\norm{z(\f{pa}(v)) - z(v)} &\leq R_{\ell-1}
&
\norm{z(\f{pa}(v')) -z(v')} &\leq R_{\ell-1}
\]
since the assigned data of $\f{pa}(v)$ is contained in the ball of radius $R_{\ell+1}$ centered at $v$ by \cref{line:update-data} of \Cref{alg:covertree}, and $z(v)$ is in the convex hull of the assigned data. 
Rearranging, we conclude that 
\[
\|z(\f{pa}(v)) - z(\f{pa}(v'))\| \leq 2R_{\ell-1} + R_{\ell} = (2 + 1/a)R_{\ell-1}.
\]
From \Cref{lem:neighbor-lemma} and since for $a>1$, $2 + 1/a \leq 2a/(a-1)$, we conclude $\f{pa}(v') \in \c{R}(\f{pa}(v))$. 
Without loss of generality, assume that $v$ was added to the tree prior to $v'$, which is unambiguous because nodes are added sequentially.
Then by \cref{line:update-data} of \Cref{alg:covertree}, and since $\f{pa}(v')\in \c{R}(\f{pa}(v))$, any node selected in \cref{line:choose-new-zeta}, and \cref{line:zeta-average} if optional Lloyd's averaging is performed, must satisfy 
\[
\norm{z(v') - z(v)} > R_\ell  
\]
which is a contradiction.
\end{proof}

\begin{lemma} \label{lem:maximum-child-nodes}
We have
\[
|\f{ch}(v)| \leq P^{\f{ext}, 1/a}_{B_{1}} \leq (2a+1)^d
.
\]
\end{lemma}
\begin{proof}
Suppose $v\in V_{\ell}$. 
As $\c{A}_\ell(v)$ is contained in the ball of radius $R_\ell$ centered at $z(v)$, for any $w \in \f{ch}(v)$, $z(w)$ is contained in this ball. 
On the other hand, since $w \in V_{\ell+1}$ and since $\f{sep}(z(V_{\ell+1})) \geq R_{\ell+1}$, balls of radius $R_{\ell+1}$ centered at $\f{ch}(v)$ form a $R_{\ell+1}$ (exterior) packing of a sphere of radius $R_{\ell+1}$. 
Since packing numbers are invariant to rescaling, this is the same as the packing number of the unit sphere by spheres of radius $1/a$.
The first inequality then follows from the definition of an exterior packing, and the second from \textcite[Corollary 4.2.13]{vershynin18}.
\end{proof}

\begin{lemma} \label{lem:maximum-r-neighbours}
We have
\[
|\f{ch}(v)| \leq P^{\f{ext}, 1/(4a)}_{B_{1}} \leq (8a+1)^d
.
\]
\end{lemma}
\begin{proof}
The proof is identical to \Cref{lem:maximum-child-nodes}, upon noting that $\c{R}(v)$ is contained in a ball of radius $4R_{\ell}$ centered at $z(v)$, as given in \cref{line:neighbor-defn} of \Cref{alg:covertree}.
\end{proof}

We are now ready to prove the main result.

\ThmCoverTree*

\begin{proof}
The claim follows by combining the above results.
\end{proof}

\section{Experimental Details}
\label{apdx:experiments}

All experiments were run on a single Nvidia V100 GPU with 32GB RAM in double precision, except for \Cref{fig:geospatial} where floating-point precision was used instead.
For the Wasserstein distance and condition number experiment of \Cref{fig:spatial-res-results}, the dataset of size $N=1000$ was generated by sampling from a GP prior with squared exponential kernel, where inputs were uniformly sampled on the $d$ dimensional cube $[-5, 5]$ with $d$ equal to \numlist{1;2;4;8}. 
The length scale of the kernel was set to $0.5\sqrt{d}$. The parameters of the clustered-data approximation were set via a cover tree with spatial resolutions \numrange[range-phrase = --]{0.05}{4.0}. Experiments for each $d$ were repeated 20 times to assess variability.

In the surface temperature geospatial data illustrative example shown of \Cref{fig:geospatial}, we ran experiments with using the sparse Gaussian process regression (SGPR) implementation from GPflow \cite{gpflow} and a GPflow-based implementation of the clustered-data Gaussian process using stochastic maximum marginal likelihood training described in \Cref{apdx:training}. 
For Hutchinson's trace estimator, we used 10 probe vectors. 
For the two-dimensional East Africa land surface temperature dataset, we split the dataset into training and testing sets of size $55884$ and $27525$, respectively. 
Both Gaussian process models were initialized with $0.1$ likelihood noise and configured with a squared exponential kernel with automatic relevance determination.
The length scales of the kernel were initialized to $1.0$. 
The inducing points were set using a cover tree with the spatial resolution \numrange[range-phrase = --]{0.03}{0.09}. 
To facilitate comparisons, in the sparse Gaussian process baseline, the inducing points were excluded from the trainable parameter set. 
For both models hyperparameters were trained using mini-batch stochastic optimization via the Adam optimizer, with constant learning rate $0.01$ and batch size $1000$.
We plot the posterior mean and standard deviation from the clustered-data Gaussian process in \Cref{fig:geospatial-high-res}.

\begin{figure}
\begin{subfigure}{\textwidth}
\begin{tabular}{>{\centering\arraybackslash} m{0.03\textwidth} @{} >{\centering\arraybackslash} m{0.47\textwidth} @{} >{\centering\arraybackslash} m{0.47\textwidth}}
\smash{\rotatebox[origin=c]{90}{$\eps = 0.09$, $M = 902$}}
&
\includegraphics[scale=0.25]{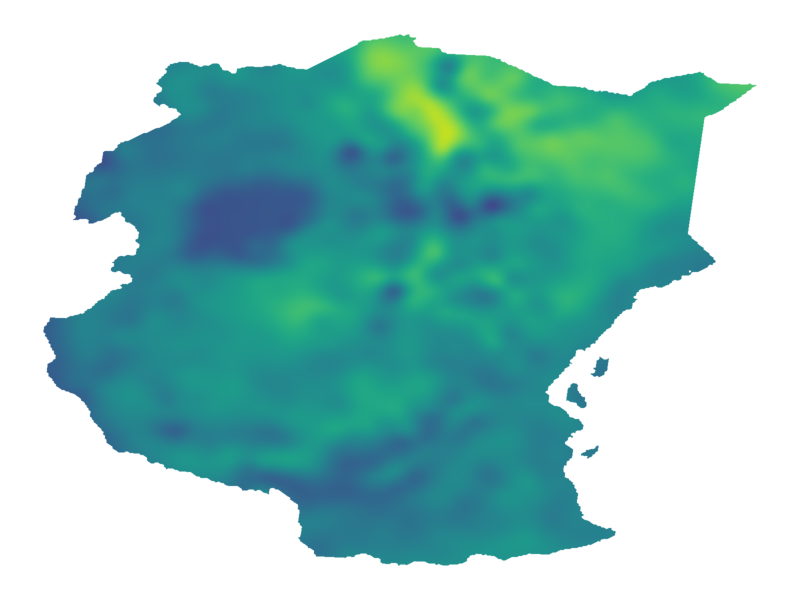}
&
\includegraphics[scale=0.25]{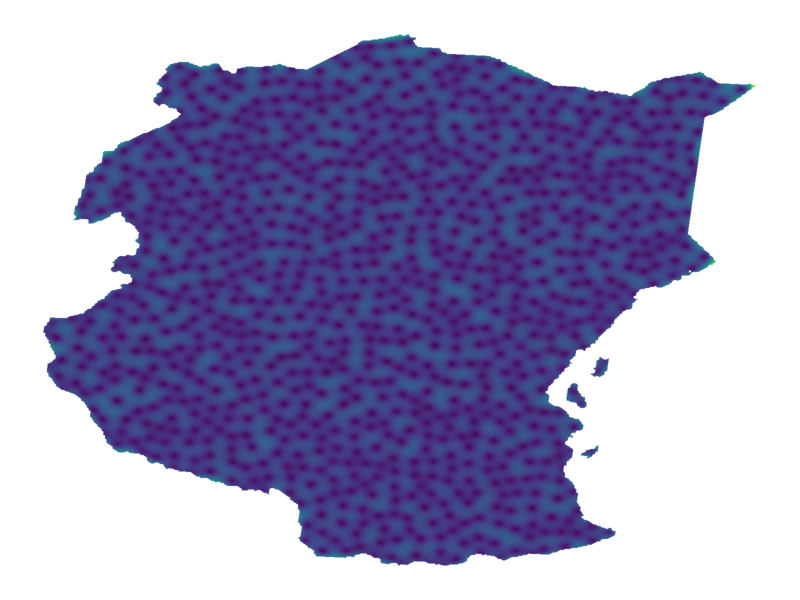}
\\
\smash{\rotatebox[origin=c]{90}{$\eps = 0.06$, $M = 1934$}}
&
\includegraphics[scale=0.25]{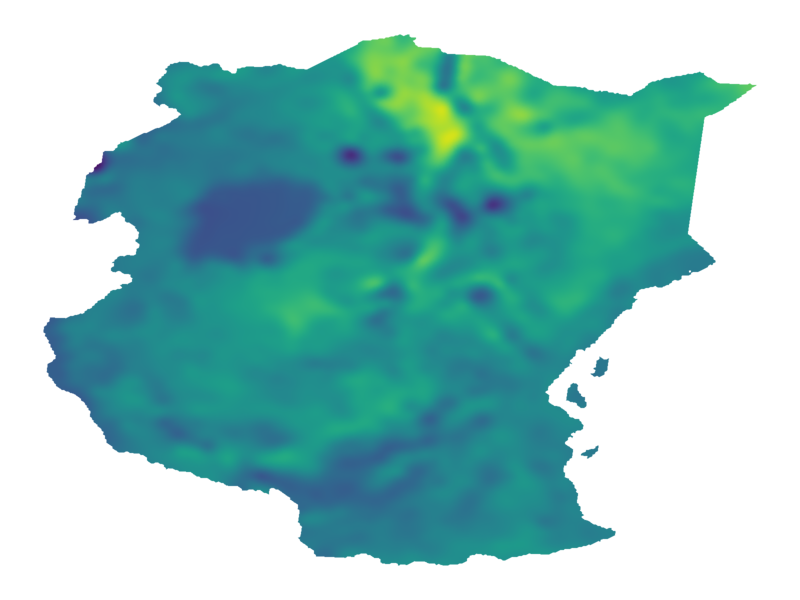}
&
\includegraphics[scale=0.25]{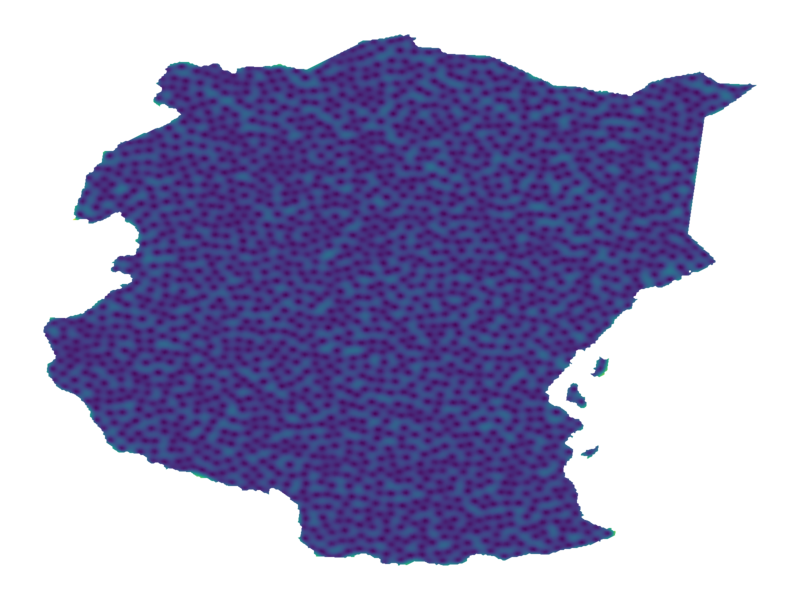}
\\
\smash{\rotatebox[origin=c]{90}{$\eps = 0.03$, $M = 6851$}}
&
\includegraphics[scale=0.25]{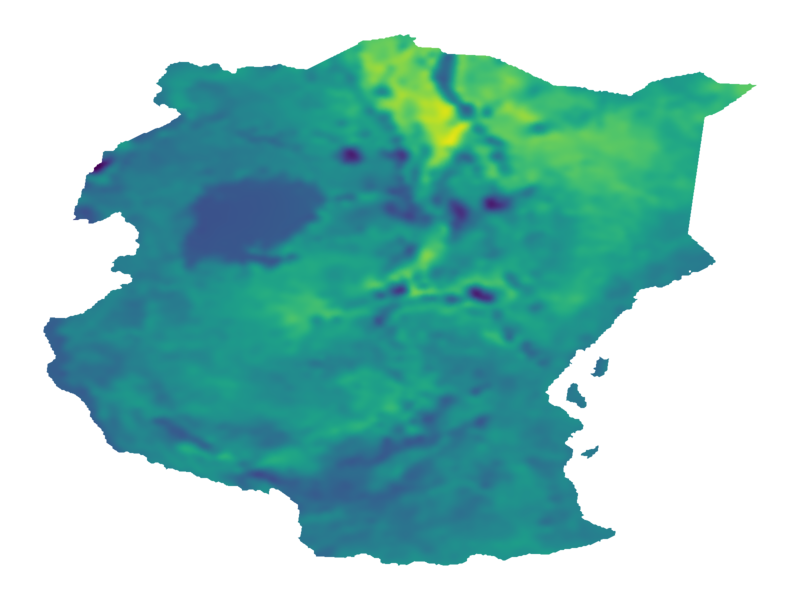}
&
\includegraphics[scale=0.25]{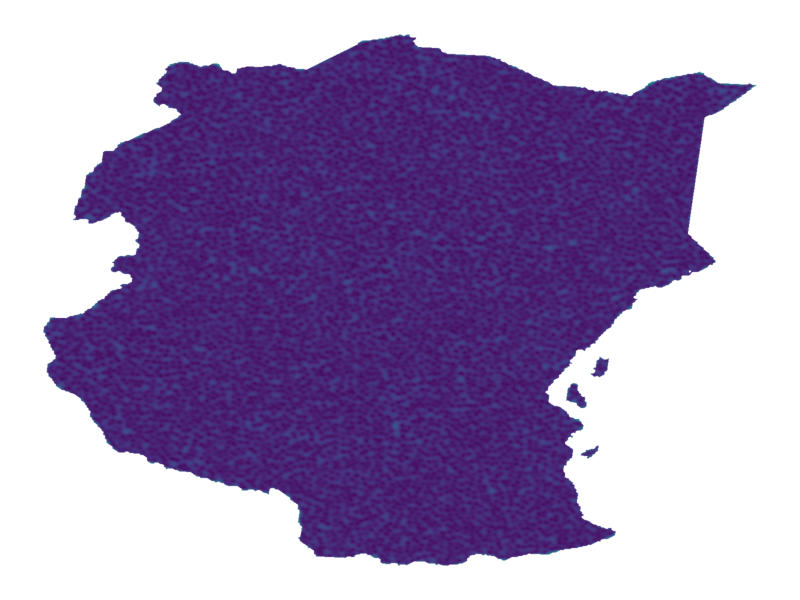}
\\
&
Mean
&
Standard Deviation
\end{tabular}
\end{subfigure}
\caption{Predictions from the clustered-data Gaussian process approximation, with varied spatial resolutions $\eps$. For each spatial resolution, we also display the resulting number of inducing points~$M$.}
\label{fig:geospatial-high-res}
\end{figure}

\end{document}